\documentclass{article}

\PassOptionsToPackage{numbers}{natbib} 

\usepackage[preprint]{neurips_2026}

\usepackage{microtype}
\usepackage{graphicx}
\usepackage{subcaption}
\usepackage{booktabs} 
\usepackage{bbm}
\usepackage{multirow}
\usepackage{hyperref}
\usepackage{paralist}

\usepackage{amsmath}
\usepackage{amssymb}
\usepackage{mathtools}
\usepackage{amsthm}
\usepackage[numbers]{natbib} 

\usepackage[capitalize,noabbrev]{cleveref}

\theoremstyle{plain}
\newtheorem{theorem}{Theorem}[section]

\newtheorem{lemma}[theorem]{Lemma}
\newtheorem{corollary}[theorem]{Corollary}
\theoremstyle{definition}

\theoremstyle{remark}

\usepackage[textsize=tiny]{todonotes}

\usepackage[utf8]{inputenc} 
\usepackage[T1]{fontenc}    
\usepackage{hyperref}       
\usepackage{url}            
\usepackage{booktabs}       
\usepackage{amsfonts}       
\usepackage{nicefrac}       
\usepackage{microtype}      
\usepackage{xcolor}         
\usepackage{caption}
\usepackage{wrapfig}
\usepackage{graphicx}
\usepackage{float}
\usepackage{algorithm}
\usepackage{algorithmic}
\title{\textsc{UMEDA}: Unified Multi-modal Efficient Data Fusion for Privacy-Preserving Graph Federated Learning via Spectral-Gated Attention and Diffusion-Based Operator Alignment}

\author{%
  Shih-Yu Lai \\
  National Taiwan University\\
  RIKEN-CCS\\
  Taipei, Taiwan; Kobe, Japan\\
  \texttt{akinesia112@gmail.com} \\
  \And
  Hirozumi Yamaguchi \\
  RIKEN-CCS\\
  The University of Osaka\\
  Osaka/Kobe, Japan\\
  \texttt{hirozumi.yamaguchi@riken.jp} \\
  \AND
  Shang-Tse Chen \\
  National Taiwan University \\
  Taipei, Taiwan \\
  \texttt{stchen@csie.ntu.edu.tw} \\
  \And
  Yu-Lun Liu \\
  National Yang Ming Chiao Tung University \\
  Hsinchu, Taiwan \\
  \texttt{yulunliu@cmlab.csie.ntu.edu.tw} \\
  \And
  Bing-Yu Chen\thanks{Corresponding author.} \\
  National Taiwan University \\
  Taipei, Taiwan \\
  \texttt{robin@ntu.edu.tw} \\
}

\begin{document}

\maketitle

\vspace{-10pt}
\begin{abstract}
Device-free localization trains models from heterogeneous wireless and visual sensors (e.g., Wi-Fi, LiDAR) distributed across edge devices. Federated learning offers a privacy-respecting framework, but is brittle when clients differ in sensor modality and resolution, when their data distributions drift, and when privacy noise destroys the structural signal needed for localization. We propose \textbf{UMEDA}, a graph federated learning framework in which clients form nodes of a global graph that share a continuous integral operator, and aggregation is reformulated as spectral signal processing on this operator. Each client encodes its local sensors with a linear-attention layer whose kernel spectrum is low-rank filtered, suppressing modality-specific residuals so clients with different sensors align in a common low-rank subspace. The server then aggregates client updates via a diffusion model over the kernel's spectral coefficients, treating updates as discretizations of a shared operator rather than topology-bound weights---this absorbs varying graph sizes and missing modalities without node-wise correspondence. To balance privacy and utility, we add an anisotropic differential-privacy mechanism that projects noise preferentially into the null space of the signal subspace, preserving dominant eigendirections while ensuring formal $(\epsilon, \delta)$-DP under gradient clipping. On MM-Fi and the RELI11D out-of-distribution benchmark, UMEDA outperforms state-of-the-art federated baselines in accuracy, convergence, and communication efficiency, particularly under high modality heterogeneity and tight privacy budgets.
\end{abstract}

\begin{wrapfigure}[17]{!t}{0.5\columnwidth}
\centering
\includegraphics[width=0.5\columnwidth]{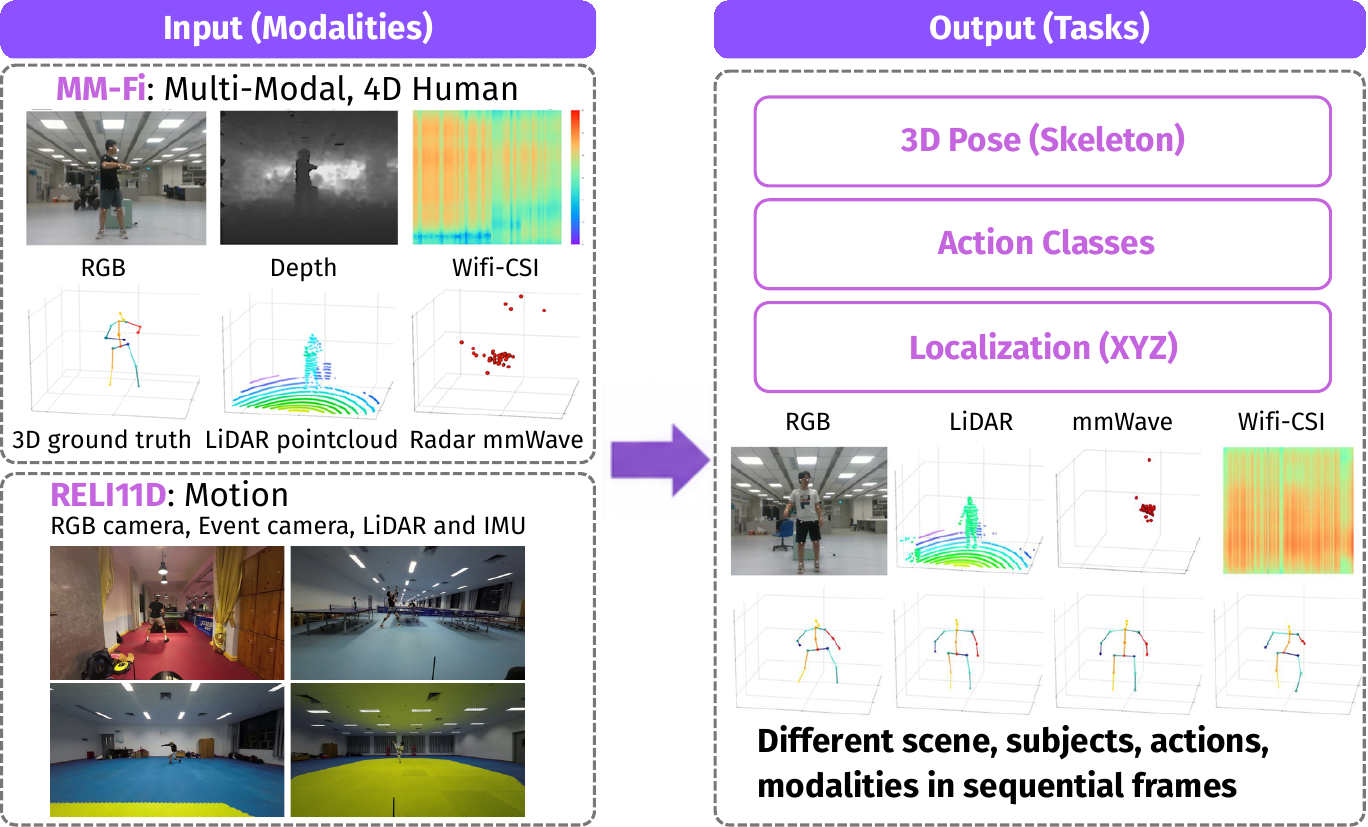}
\caption{Overview of MM-Fi and RELI11D datasets, input modalities, and target tasks for multi-modal human sensing under heterogeneous scenes, subjects, actions, and sensors.}
\label{fig:arch_overview}
\end{wrapfigure}

\newpage
\section{Introduction}
\vspace{-8pt}

Device-free localization (DFL) infers a person's position, pose, and activity from ambient sensor streams—Wi-Fi channel state information~\citep{JiangMobiCom20,Wang2022MetaFiPlusPlus,MFFALoc,LocFree}, millimeter-wave radar~\citep{ChenMM22,LeeWACV23}, and LiDAR or depth point clouds~\citep{Wang2018,YenCVPR24}—without requiring users to wear any device. Fusing these modalities is attractive because each sensor compensates for the others' blind spots~\citep{YangNeurIPS23,AnNeurIPS22,chen2023immfusion}, but it raises a problem that is fundamentally not centralized: the raw streams reveal where a person is, what they are doing, and what their environment looks like, and they are too bandwidth-heavy to ship to a central server. Federated learning (FL)~\citep{McMahan2017FedAvg} is the natural framework, but standard FL does not match three properties of multi-modal sensing.

\textit{Clients have different sensors and different resolutions.} A Wi-Fi client and a LiDAR client do not produce comparable feature tensors, and even within one modality the spatial resolution varies across devices. Generic FL methods such as FedAvg~\citep{McMahan2017FedAvg}, FedProx~\citep{Li2020FedProx}, and SCAFFOLD~\citep{Karimireddy2020SCAFFOLD} aggregate by averaging weights and assume every client trains the same architecture on the same input shape. Graph federated learning (GFL)~\citep{Chen2023FedGraphSurvey,Li2024FederatedGNNs}—where clients are treated as nodes of a global graph and aggregation is structure-aware~\citep{FedGTA,FedTAD,Qiu2022FedGraphNN}—relaxes the homogeneity assumption, but existing GFL methods build the inter-client graph from heuristic similarity scores that become unstable when client A produces Wi-Fi spectrograms and client B produces sparse LiDAR point clouds. There is no representation in which the two clients' updates can be directly compared.

\textit{Client distributions drift.} Different rooms, body sizes, and sensor placements shift the joint distribution of inputs and labels. Optimizer-based fixes~\citep{Li2020FedProx,Karimireddy2020SCAFFOLD,Reddi2021FedOpt,Acar2021FedDyn,Li2021MOON} constrain weight-space updates but do not align the underlying representations. Recent generative FL approaches~\citep{weng2024feddiff,pinaya2024feddm,chen2025fedbip} use diffusion models to synthesize alignment data on the server, but operate in fixed-dimensional Euclidean space and do not transfer to graphs of varying size. Neural operators~\citep{kovachki2021neuraloperator,li2020mgno} provide a resolution-independent function-space alternative, yet have been used for supervised PDE regression rather than for federated distribution alignment.

\textit{Differential privacy is misaligned with the threat model.} In DFL, the server is allowed to know that a client is participating, but should not be able to reconstruct what that client looks like, what action the person performs, or where they are in the room. Standard DP mechanisms~\cite{Geyer2017,Dwork2014DP,Abadi2016MomentsAccountant,andrew2022adaptiveclipping} clip updates and inject isotropic Gaussian noise; this protects all directions equally, including the low-rank directions that carry the localization signal. The result is that strong privacy budgets erase the very structure needed for accurate pose and position estimation.

\textit{Our approach.} The three obstacles share a root cause: weight-space federation presumes clients agree on architecture, input shape, and update geometry—none of which holds here. This raises our central question: \textit{``Can multi-modal sensing be federated in a representation invariant to which sensor each client holds, while preserving formal $(\epsilon,\delta)$-DP guarantees?''} We answer with \textit{UMEDA}, a graph federated learning framework where clients are nodes sharing a continuous integral operator, and aggregation is reformulated as spectral signal processing on that operator. The hypothesis: while sensors discretize the world differently, the underlying physical kernel is shared, so federation should happen in operator space, not weight space. UMEDA contributes:

\begin{itemize}
    \item \textbf{Spectral-Gated Linear Transformer (SGLT).} A linear-attention encoder that performs a learned spectral gate on the singular spectrum of its attention kernel. The gate keeps the low-rank directions that carry shared semantics across clients and attenuates the high-frequency directions where modality-specific noise concentrates, so clients with different sensors are encoded into a common subspace.
    
    \item \textbf{Diffusion-based Graph Neural Operator (Diff-GNO).} A server-side aggregator that views each client's update as a discretization of a shared continuous integral kernel—a graph neural operator—and aggregates via a diffusion model trained on the operator's spectral coefficients. Generative denoising replaces weight averaging, so the consensus is well defined even when clients have different graph sizes or miss whole modalities.
    
    \item \textbf{Subspace-Projected Differential Privacy (SP-DP).} An anisotropic DP mechanism that projects Gaussian noise onto the null space of a public signal subspace before adding it to the clipped client update. This hides client-specific scene appearance, motion, and position with formal $(\epsilon,\delta)$ guarantees, while leaving the dominant eigendirections—those carrying the localization signal—almost untouched.
\end{itemize}

On MM-Fi~\citep{YangNeurIPS23} and the RELI11D out-of-distribution benchmark~\citep{YenCVPR24}, UMEDA outperforms state-of-the-art federated, graph-FL, and DP-FL baselines in accuracy, convergence, and communication, particularly under high modality heterogeneity and tight privacy budgets. Figure~\ref{fig:arch_overview} shows the multi-modal human sensing datasets, input modalities, and target tasks; Figure~\ref{fig:modules} summarizes the end-to-end pipeline and three core components.

\section{Related Work}
\vspace{-8pt}
\textit{Graph learning for heterogeneous signals.}
Unifying multi-modal sensors (LiDAR, Wi-Fi CSI, mmWave) is fundamentally about reconciling \textit{discretization heterogeneity}. Voxelization and multi-view projection discard fine topology; point-based GNNs such as Point-GNN~\cite{Shi2020PointGNN} and Graph Attention Convolution~\cite{GraphAttentionConvolution} build $k$-NN graphs to capture local geometry, but assume homogeneous node features and fail to align modalities of differing densities and noise. Graph Transformers~\cite{Wang2022GTNet,2023SuperpointTransformer} model long-range dependencies but scale quadratically and amplify high-frequency sensor noise. Linear-attention variants (Performer~\cite{choromanski2021performer}, Nystr\"omformer~\cite{xiong2021nystromformer}, Linformer~\cite{wang2020linformer}) reduce cost but lack \textit{spectral filtering} inductive biases—they treat all spectral components equally and cannot separate the shared low-rank semantic subspace (human pose) from modality-specific high-frequency artifacts (sensor noise), motivating an explicit spectral gate. Recent graph generative work further underscores the need to handle variable-size/discrete structures (continuous-time discrete-state graph diffusion~\cite{xu2024dsct_graphdiff}, diffusion-free next-scale generation~\cite{duarte2025magra}).

\textit{Federated learning under non-IID concept drift.}
Optimization-based FL such as FedProx~\cite{Li2020FedProx} and SCAFFOLD~\cite{Karimireddy2020SCAFFOLD} constrains local updates in \emph{weight space} and struggles with the severe concept drift induced by environmental variation in sensing. \textit{Generative FL} trains or leverages diffusion models for server-side alignment~\cite{weng2024feddiff,pinaya2024feddm}, but standard diffusion operates in fixed-dimensional Euclidean space and is ill-suited to graphs whose node count and topology vary across clients—prompting graph-specific formulations~\cite{xu2024dsct_graphdiff}. \textit{Neural operators}~\cite{kovachki2021neuraloperator} provide a resolution-independent alternative by learning maps between function spaces, with graph extensions via GNO~\cite{li2020mgno}; however, GNOs have been used for supervised PDE regression, leaving their potential for \textit{generative distribution alignment} across heterogeneous clients underexplored.

\textit{Privacy--utility trade-offs.}
DP-SGD~\cite{Geyer2017} clips gradients and injects \emph{isotropic} Gaussian noise. In geometric deep learning this is harmful: the structural utility of a graph (e.g., localization geometry) concentrates in the \textit{dominant eigenspace}, so uniform perturbation yields a poor privacy-utility frontier. Topology-aware mechanisms~\cite{Chen2023FedGraphSurvey} adjust noise by edge existence or node degree but target structural anonymity (hiding links) rather than the \textit{semantic utility} of node features. Missing is a mechanism that uses \textit{subspace projection} to confine noise to the signal's null space, masking sensitive updates without destroying the low-rank geometry essential for localization and pose estimation.

\textit{Graph federated learning (GFL) via topology-aware aggregation.}
GFL models clients as nodes of a global graph with edges encoding statistical or physical dependence, replacing parameter averaging with structure-aware fusion~\cite{Chen2023FedGraphSurvey,Li2024FederatedGNNs}. Topology-aware methods such as FedGTA and FedTAD~\cite{FedGTA,FedTAD} learn this inter-client graph to re-weight contributions under non-IID; benchmarks FedGraphNN~\cite{Qiu2022FedGraphNN} and OpenFGL~\cite{Zhang2023OpenFGL} standardize evaluation. These methods rely on \emph{heuristic} similarity metrics that become unstable under \textit{multi-modal heterogeneity} (Client A on Wi-Fi vs.\ Client B on LiDAR) and lack a continuous function-space alignment. UMEDA replaces heuristic graph construction with a \textit{Diffusion-based Graph Neural Operator} that generatively models client relationships, yielding topology-aware aggregation that is robust even with disjoint modalities and varying resolutions.

\section{Method}
\label{sec:method}
\vspace{-8pt}
\begin{figure*}[t]
\centering
\includegraphics[width=\textwidth]{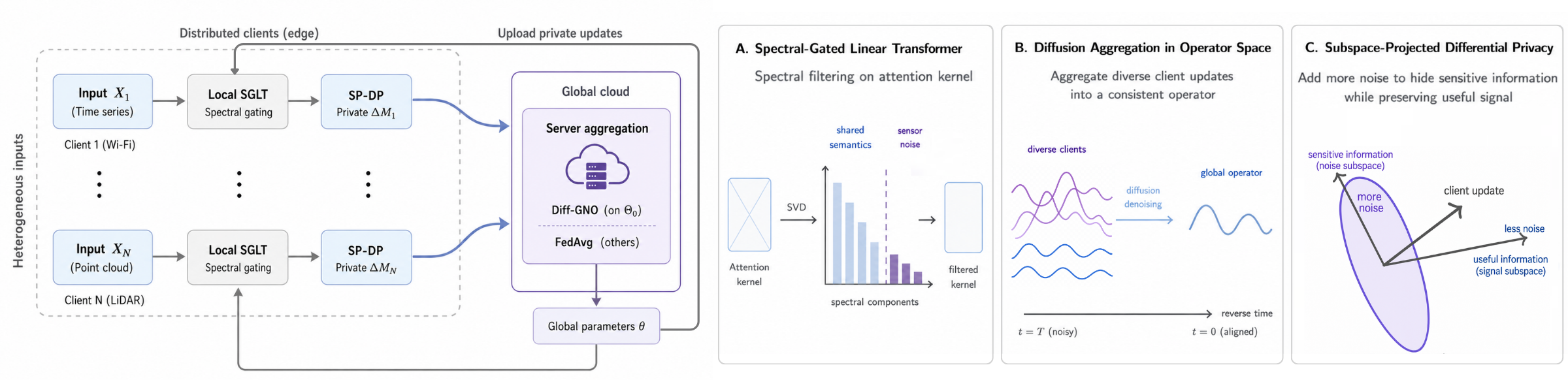}
\caption{Architecture and mechanisms of UMEDA. 
\textbf{Left:} end-to-end pipeline. Heterogeneous clients (Wi-Fi time series, LiDAR point clouds, etc.) apply SGLT to produce kernel updates $\Delta\mathbf{M}_k$, privatize them via SP-DP, and upload to the server, which aggregates with Diff-GNO on $\mathbf{\Theta}_0=\mathrm{vec}(\Delta\tilde{\mathbf{M}})$ and FedAvg on the rest, then broadcasts $\theta$.
\textbf{Right:} \textit{(A) SGLT} retains low-rank shared semantics (blue bars) and suppresses high-frequency sensor noise (purple bars) on the attention kernel.
\textit{(B) Diff-GNO} denoises diverse client operator updates ($t{=}T$, noisy) into a globally aligned operator ($t{=}0$) via reverse-time diffusion.
\textit{(C) SP-DP} adds more noise in the null (sensitive) subspace and less in the signal (useful) subspace, preserving utility under $(\epsilon,\delta)$-DP.}
\label{fig:modules}
\end{figure*}

\textit{Overview.}
We formulate the multi-modal Device-Free Localization (DFL) problem within a GFL framework. Let  $\mathcal{C} = \{1, \dots, N\}$ be the set of clients. Each client $k \in \mathcal{C}$ observes a local dataset $\mathcal{D}_k = \{(\mathbf{X}_k^{(i)}, \mathbf{Y}_k^{(i)})\}_{i=1}^{n_k}$, where $\mathbf{X}_k^{(i)}$ represents heterogeneous sensor inputs (e.g., Wi-Fi CSI, LiDAR point clouds) and $\mathbf{Y}_k^{(i)}$ represents the target labels (e.g., 3D coordinates, action classes).
A fundamental challenge is \textit{discretization heterogeneity}:  $\mathbf{X}_k$ resides in different topological spaces across clients (e.g., $\mathbb{R}^{N_{pts} \times 3}$ for LiDAR vs. $\mathbb{R}^{N_{sub} \times T}$ for Wi-Fi). We propose \textit{UMEDA}, a framework that maps these heterogeneous inputs into a unified \textit{operator space}---formally a Hilbert--Schmidt space of integral operators acting on signals~\citep{kovachki2021neuraloperator}---decoupling the federated representation from the discretization that any single client happens to use. We parameterize the global graph operator as $\mathcal{G}_\theta\in\mathfrak{F}$ with parameters $\theta$, where $\mathfrak{F}$ denotes this Hilbert--Schmidt operator hypothesis space.
The federated objective is:
\begin{equation}
\theta^* \;=\; \arg\min_{\theta}\; \mathbb{E}_{k\sim\mathcal{P}}
\left[\mathcal{L}_k(\mathcal{G}_\theta;\mathcal{D}_k)\right].
\label{eq:fed_obj}
\end{equation}
where $\mathfrak{F}$ is the hypothesis space of graph operators. UMEDA achieves this via three coupled components: (1) Spectral-Gated Linear Transformer (SGLT) for local alignment, (2) Diffusion-based Graph Neural Operator (Diff-GNO) for global aggregation, and (3) Subspace-Projected Differential Privacy (SP-DP).

\subsection{Spectral-Gated Linear Transformer (SGLT)}
\label{sec:sglt}

To reconcile multi-modal heterogeneity, we define a local encoder $\Phi_k$ that maps discrete inputs $\mathbf{X}_k$ to a latent sequence $\mathbf{H}_k \in \mathbb{R}^{L \times d}$. Standard self-attention suffers from quadratic complexity and overfits to high-frequency sensor noise. We propose \textit{SGLT}, which incorporates a low-rank spectral filtering mechanism.

\textit{Linearized Attention via Kernel Approximation.}
Let $\mathbf{Q}, \mathbf{K}, \mathbf{V} \in \mathbb{R}^{L \times d}$ be the query, key, and value matrices. We utilize a linearized attention mechanism via a feature map:
\begin{equation}
\text{SGLT}(\mathbf{Q}, \mathbf{K}, \mathbf{V}) = \phi(\mathbf{Q}) \left( \phi(\mathbf{K})^\top \mathbf{V} \right).
\end{equation}
The term $\mathbf{M} = \phi(\mathbf{K})^\top \mathbf{V} \in \mathbb{R}^{d \times d}$ encapsulates the global semantic structure. In multi-modal settings, modality-specific noise (e.g., Wi-Fi multipath, LiDAR outliers) manifests as high-frequency perturbations in $\mathbf{M}$.

\textit{Spectral Gating Mechanism.}
We compute the SVD 
\begin{equation}
\mathbf{M}=\mathbf{U}\mathbf{\Sigma}\mathbf{W}^\top.
\label{eq:svd}
\end{equation}
During training, we use a differentiable relaxation
\small{
\begin{equation}
g(\sigma_i) \;=\; \mathrm{sigmoid}\!\left(\frac{\sigma_i-\tau}{\beta}\right), 
g(\mathbf{\Sigma}) := \mathrm{diag}(g(\sigma_1),\dots,g(\sigma_d)),
\label{eq:soft_gate}
\end{equation}
}
and reconstruct
\begin{equation}
\hat{\mathbf{M}} \;=\; \mathbf{U}\big(g(\mathbf{\Sigma})\odot \mathbf{\Sigma}\big)\mathbf{W}^\top.
\label{eq:filteredM}
\end{equation}
For analysis and ablations, we also consider the hard threshold gate
$g(\sigma_i)=\mathbb{I}(\sigma_i>\tau)$, which recovers a truncated SVD.

\textit{Output.}
We apply the filtered semantic matrix $\hat{\mathbf{M}}$ with the standard linear-attention normalization:
\begin{equation}
\mathbf{H}' \;=\;
\frac{\phi(\mathbf{Q})\,\hat{\mathbf{M}}}
{\phi(\mathbf{Q})\left(\phi(\mathbf{K})^\top\mathbf{1}\right)+\varepsilon},
\label{eq:sglt_out_norm}
\end{equation}
where $\mathbf{1}\in\mathbb{R}^{L}$ is an all-ones vector and $\varepsilon>0$ ensures stability.

\subsection{Subspace-Projected Differential Privacy (SP-DP)}
\label{sec:sp_dp}

To address the privacy-utility trade-off while ensuring rigorous guarantees, we propose \textit{Safe SP-DP}. Unlike standard DP mechanisms that perturb all parameters isotropically, we decouple the subspace estimation from the private gradient update to avoid privacy leakage through the projection basis.

\textit{Safe Subspace Projection.} 
Instead of using the client's private local sensitivity basis, we leverage the \textit{global} semantic operator from the previous round, $\theta_M^t$, to define the projection manifold. Let $\theta_M^t = \mathbf{U}_{g} \mathbf{\Sigma}_{g} \mathbf{W}_{g}^\top$ be the SVD of the current global semantic operator. We define the public signal projector $\mathbf{P}_{\mathcal{S}} = \mathbf{U}_{g, 1:r} \mathbf{U}_{g, 1:r}^\top$ based on the top-$r$ dominant eigenspace.

The client's update $\Delta\mathbf{M}_k$ is privatized in three steps. First, let $\Delta\mathbf{m} := \mathrm{vec}(\Delta\mathbf{M}_k) \in \mathbb{R}^{d^2}$. We bound the update sensitivity via clipping:
\begin{equation}
\Delta\bar{\mathbf{m}}
\;=\;
\Delta\mathbf{m}\cdot \min\left\{1,\frac{C}{\|\Delta\mathbf{m}\|_2}\right\},
\label{eq:clip_vec}
\end{equation}
where $C$ is the clipping threshold. Second, we inject anisotropic Gaussian noise projected onto the global signal/null subspaces:
\begin{equation}
    \Delta \tilde{\mathbf{m}} = \Delta\bar{\mathbf{m}} + (\mathbf{P}_{\mathcal{S}} \otimes \mathbf{I}) \mathbf{z}_{\mathrm{sig}} + (\mathbf{I} - \mathbf{P}_{\mathcal{S}} \otimes \mathbf{I}) \mathbf{z}_{\mathrm{null}},
    \label{eq:spdp_vec}
\end{equation}
where $\mathbf{z}_{\mathrm{sig}} \sim \mathcal{N}(0, \sigma_{\mathrm{sig}}^2 \mathbf{I})$ and $\mathbf{z}_{\mathrm{null}} \sim \mathcal{N}(0, \sigma_{\mathrm{null}}^2 \mathbf{I})$ are independent noise vectors. Finally, to ensure $(\epsilon, \delta)$-differential privacy, we calibrate the base noise scale $\sigma_{\mathrm{sig}}$ according to the Gaussian mechanism:
\begin{equation}
\sigma_{\mathrm{sig}} \;\ge\; \frac{C\sqrt{2\ln(1.25/\delta)}}{\epsilon}.
\label{eq:dp_calib}
\end{equation}
Since $\mathbf{P}_{\mathcal{S}}$ is derived strictly from public broadcast information ($\theta_M^t$) and $\sigma_{\mathrm{null}} \ge \sigma_{\mathrm{sig}}$, the projection acts as post-processing (or adds strictly more noise in the null space), thereby preserving the formal privacy guarantee established by $\sigma_{\mathrm{sig}}$.

\subsection{Diffusion-based Graph Neural Operator}
\label{sec:diff_gno}

Standard federated aggregation (e.g., FedAvg) often fails under non-IID drift. To address this, we propose the \textit{Diffusion-based Graph Neural Operator (Diff-GNO)}, which reformulates aggregation as learning a generative distribution over client operators.

However, applying diffusion directly to the full operator update $\Delta \mathbf{M} \in \mathbb{R}^{d \times d}$ suffers from the curse of dimensionality, particularly when the number of participating clients is small ($|\mathcal{S}_t| \ll d^2$). To make Diff-GNO computationally tractable and statistically robust, we implement it via \textit{Spectral-Latent Diffusion}, which performs generative aggregation solely on the low-dimensional spectral signature of the updates.

\textit{Spectral Projection.}
Given the privatized update $\Delta \tilde{\mathbf{M}}_k$, we first project it onto the global basis $\mathbf{U}_g$ and $\mathbf{W}_g$ (computed from the global operator $\theta_M^t$) to obtain the spectral coefficient matrix $\mathbf{\Sigma}^{(k)} \in \mathbb{R}^{r \times r}$:
\begin{equation}
    \mathbf{\Sigma}^{(k)} = \mathbf{U}_{g, 1:r}^\top (\Delta \tilde{\mathbf{M}}_k) \mathbf{W}_{g, 1:r}.
\end{equation}
We define the diffusion state as the flattened spectral coefficients $\mathbf{\Theta}_0^{(k)} := \mathrm{vec}(\mathbf{\Sigma}^{(k)}) \in \mathbb{R}^{r^2}$. By choosing $r \ll d$ (e.g., $r=16, d=256$), the dimensionality is reduced by orders of magnitude ($65,536 \to 256$), allowing the score model to capture the complex distribution of client updates even with limited samples.

\textit{Forward VE-SDE.}
We adopt a variance-exploding diffusion on the operator-update space:
\begin{equation}
d\mathbf{\Theta}_t = g(t)\, d\mathbf{W}_t,\qquad t\in[0,1],
\label{eq:ve_forward}
\end{equation}
which yields $\mathbf{\Theta}_t=\mathbf{\Theta}_0+\sigma(t)\boldsymbol{\epsilon}$ with
$\boldsymbol{\epsilon}\sim\mathcal{N}(0,\mathbf{I})$ and $\sigma^2(t)=\int_0^t g^2(\tau)\,d\tau$.

\textit{Denoising Score Matching (DSM).}
We train a score model $s_\psi(\mathbf{\Theta}_t,t)\approx\nabla_{\mathbf{\Theta}_t}\log p_t(\mathbf{\Theta}_t)$ via
\small{
\begin{equation}
\mathcal{L}_{\mathrm{DSM}}(\psi)=
\mathbb{E}_{t,\mathbf{\Theta}_0,\boldsymbol{\epsilon}}
\left[\left\|s_\psi(\mathbf{\Theta}_t,t)+\frac{\boldsymbol{\epsilon}}{\sigma(t)}\right\|_2^2\right],
\mathbf{\Theta}_t=\mathbf{\Theta}_0+\sigma(t)\boldsymbol{\epsilon}.
\label{eq:dsm}
\end{equation}
}

\textit{Reverse-time SDE (aggregation).}
Given $s_\psi$, the reverse-time dynamics are
\begin{equation}
d\mathbf{\Theta}_t = -g(t)^2 s_\psi(\mathbf{\Theta}_t,t)\,dt + g(t)\, d\bar{\mathbf{W}}_t.
\label{eq:ve_reverse}
\end{equation}

Finally, the high-dimensional operator update is reconstructed as $\Delta \mathbf{M}^* = \mathbf{U}_{g, 1:r} \mathrm{unvec}(\mathbf{\Theta}^*) \mathbf{W}_{g, 1:r}^\top$. This effectively aligns the distribution of client updates while avoiding overfitting in the high-frequency residual space.

To apply this consensus, we partition the global parameters as $\theta=(\theta_{\setminus M},\theta_M)$, where $\theta_M$ denotes the SGLT semantic block and $\theta_{\setminus M}$ collects all remaining parameters (encoder, projections, heads). The server updates the semantic operator $\theta_M$ using the Diff-GNO consensus $\Delta\mathbf{M}^*$, while aggregating the remaining parameters $\theta_{\setminus M}$ via standard FedAvg.

\begin{wrapfigure}[38]{r}{0.5\columnwidth}
\vspace{-30pt}
\centering
%
\begin{minipage}[t]{0.49\columnwidth}
\begin{algorithm}[H]
\caption{UMEDA}
\label{alg:umeda}
\footnotesize
\begin{algorithmic}[1]
\STATE \textbf{Input:} clients $\mathcal{C}$; rounds $T$; steps $E$; clip $C$; privacy $(\epsilon,\delta)$; rank $r$; VE schedule $g(t)$.
\STATE \textbf{Init:} parameters $\theta^0=(\theta_{\setminus M}^0,\theta_M^0)$, score network $\psi^0$.
\FOR{$t=0,\dots,T-1$}
  \STATE \textbf{Global Basis:} Server computes SVD $\theta_M^t \approx \mathbf{U}_g \mathbf{\Sigma}_g \mathbf{W}_g^\top$; samples $\mathcal{S}_t$ and broadcasts $\theta^t, \mathbf{U}_g, \mathbf{W}_g$.
  \FORALL{$k\in\mathcal{S}_t$ \textbf{in parallel}}
    \STATE $\Delta\mathbf{M}_k, \Delta\theta_{\setminus M}^{(k)} \leftarrow \text{LocalTrain}(\theta^t, E)$. \textit{// Discretized kernel update}
    \STATE $\Delta\tilde{\mathbf{M}}_k \leftarrow \mathrm{SP\mbox{-}DP}(\Delta\mathbf{M}_k, \mathbf{U}_g; C, \epsilon, \delta)$. \textit{// Safe projection using global basis}
    \STATE Upload privatized kernel $\Delta\tilde{\mathbf{M}}_k$ and non-kernel update $\Delta\theta_{\setminus M}^{(k)}$.
  \ENDFOR
  \STATE \textbf{Spectral Projection:} Server projects updates to latent spectrum $\mathbb{R}^{r^2}$:\\
  \quad $\mathbf{\Theta}_0^{(k)} \leftarrow \mathrm{vec}(\mathbf{U}_{g, 1:r}^\top \Delta\tilde{\mathbf{M}}_k \mathbf{W}_{g, 1:r}), \quad \forall k \in \mathcal{S}_t$.
  \STATE \textbf{Diff-GNO:} Update score $s_\psi$ on $\{\mathbf{\Theta}_0^{(k)}\}$; sample consensus $\mathbf{\Theta}^* \in \mathbb{R}^{r^2}$ via reverse VE-SDE.
  \STATE \textbf{Update:} Reconstruct kernel and aggregate remaining blocks:\\
  \quad $\theta_M^{t+1} \leftarrow \theta_M^t + \eta_s (\mathbf{U}_{g, 1:r} \mathrm{unvec}(\mathbf{\Theta}^*) \mathbf{W}_{g, 1:r}^\top)$.\\
  \quad $\theta_{\setminus M}^{t+1} \leftarrow \theta_{\setminus M}^t + \eta_s \text{FedAvg}(\{\Delta\theta_{\setminus M}^{(k)}\})$.
\ENDFOR
\STATE \textbf{Return:} $\theta^T$.
\end{algorithmic}
\end{algorithm}
\end{minipage}

\vspace{-6pt}

\begin{minipage}[t]{0.49\columnwidth}
\begin{algorithm}[H]
\caption{Safe SP-DP: Subspace-projected DP with Global Basis Anchor}
\label{alg:spdp}
\footnotesize
\begin{algorithmic}[1]
\STATE \textbf{Input:} Update $\Delta\mathbf{M}$; Global basis $\mathbf{U}_g$; clip $C$; privacy $(\epsilon,\delta)$; rank $r$; $\sigma_{\mathrm{null}}\ge\sigma_{\mathrm{sig}}$.
\STATE Vectorize: $\Delta\mathbf{m}\leftarrow \mathrm{vec}(\Delta\mathbf{M})\in\mathbb{R}^{d^2}$
\STATE Clip: $\Delta\bar{\mathbf{m}}\leftarrow \Delta\mathbf{m}\cdot \min\{1,\,C/\|\Delta\mathbf{m}\|_2\}$
\STATE \textbf{Projectors:} $\mathbf{P}_{\mathcal{S}}\leftarrow \mathbf{U}_{g, 1:r}\mathbf{U}_{g, 1:r}^\top$; $\mathbf{P}_{\mathcal{S}^\perp}\leftarrow \mathbf{I}-\mathbf{P}_{\mathcal{S}}$
\STATE Calibrate: $\sigma_{\mathrm{sig}}\leftarrow \frac{C\sqrt{2\ln(1.25/\delta)}}{\epsilon}$
\STATE Sample $\mathbf{z}_{\mathrm{sig}}\sim\mathcal{N}(0,\sigma_{\mathrm{sig}}^2\mathbf{I})$, $\mathbf{z}_{\mathrm{null}}\sim\mathcal{N}(0,\sigma_{\mathrm{null}}^2\mathbf{I})$
\STATE Noise injection:
$\Delta\tilde{\mathbf{m}}\leftarrow
\Delta\bar{\mathbf{m}}+
(\mathbf{P}_{\mathcal{S}}\otimes \mathbf{I})\mathbf{z}_{\mathrm{sig}}+
(\mathbf{P}_{\mathcal{S}^\perp}\otimes \mathbf{I})\mathbf{z}_{\mathrm{null}}$
\STATE Unvectorize: $\Delta\tilde{\mathbf{M}}\leftarrow \mathrm{unvec}(\Delta\tilde{\mathbf{m}})$
\STATE \textbf{Return:} $\Delta\tilde{\mathbf{M}}$
\end{algorithmic}
\end{algorithm}
\end{minipage}

\vspace{-4pt}
\end{wrapfigure}

\textit{Summary.}
\textit{Alg.~\ref{alg:umeda} }summarizes the end-to-end UMEDA training loop: each round, selected clients run SGLT-based local training to produce a semantic-kernel update $\Delta\mathbf{M}_k$ and the remaining parameter updates, privatize $\Delta\mathbf{M}_k$ via SP-DP (\textit{Alg.~\ref{alg:spdp}}), and the server performs diffusion-based aggregation (Diff-GNO) on the privatized kernel updates while aggregating the non-kernel parameters by FedAvg.

\section{Experiments}
\vspace{-8pt}
\subsection{Experimental Setup}
\label{sec:setup}


\textit{Datasets.}
\textit{MM-Fi} \cite{YangNeurIPS23} is a large-scale multi-modal human sensing benchmark with synchronized heterogeneous streams
(e.g., Wi-Fi CSI and depth/radar/LiDAR-like modalities). We follow the standard protocol of treating each sample as a multi-modal observation and construct graph/token inputs $G$ by modality-specific discretization, explicitly inducing varying graph sizes and feature topologies across modality configurations. To validate fidelity of our reproduction, we cross-checked Wi-Fi CSI unimodal pose estimation against the original MM-Fi benchmark~\citep{YangNeurIPS23} and HPE-Li, reproducing their reported MPJPE within $\pm 2.5$\,mm under matched splits (Appendix~\ref{app:wifi_validation}).
\textit{RELI11D}~\citep{YenCVPR24} provides a complementary set of modalities---RGB camera, event camera, LiDAR, and IMU---captured across diverse indoor/outdoor scenes. Since RELI11D shares \emph{no} sensor with the source MM-Fi clients (which use Wi-Fi CSI, depth, and mmWave), MM-Fi$\to$RELI11D constitutes a strict zero-shot cross-modality transfer test rather than an in-modality domain shift, directly probing whether federation in operator space generalizes across disjoint sensor stacks. We evaluate transfer by training on MM-Fi clients and testing on RELI11D domains without adaptation.

\textit{Federated simulation.}
We simulate $K=100$ clients. Each client receives a private subset of samples and a client-specific modality set $\mathcal{M}_k$.
To isolate different failure modes:
\begin{itemize}
  \item \textit{Non-IID labels:} we partition by a Dirichlet distribution over labels with concentration $\alpha$,
  where smaller $\alpha$ yields stronger statistical heterogeneity.
  \item \textit{Discretization heterogeneity:} clients use different modality/resolution configurations.
  Concretely, we instantiate three client types:
  \emph{Type A} (coarser / sparse modality discretization),
  \emph{Type B} (denser / richer modality discretization),
  \emph{Type C} (mixed modalities with occasional missing streams).
\end{itemize}

\textit{Tasks and metrics.}
We consider three representative targets consistent with device-free sensing:
(1) \textbf{3D pose estimation} measured by MPJPE (mm, lower is better),
(2) \textbf{action recognition} measured by Top-1 accuracy (\%, higher is better),
(3) \textbf{localization} measured by RMSE (cm, lower is better).
For privacy, we report performance under multiple $(\epsilon,\delta)$ budgets and keep $\delta$ fixed.

\textit{Training protocol.}
We run up to $T=1000$ communication rounds with $K=100$ simulated clients and a fixed client sampling rate $q=0.4$.
For efficiency experiments (Tab.~\ref{tab:comm}), we report the number of rounds needed to reach a fixed target (Loc RMSE $\le 12.5$) with early stopping.
Each selected client performs $E=100$ local gradient steps with batch size $B=64$ using AdamW
(lr $\eta=10^{-3}$, weight decay $10^{-4}$).
Unless stated otherwise, we use server step size $\eta_s=1.0$ (FedAvg-style averaging for non-SGLT blocks; Diff-GNO for the SGLT block).
We report the mean over three random seeds.

\textit{Models.}
SGLT uses hidden dimension $d=256$ with $8$ heads.
For linear attention, we instantiate $\phi(\cdot)$ with FAVOR+ random features (positive random features for softmax-kernel approximation)
using $m=256$ random features per head, and apply the standard normalization in Eq.~\eqref{eq:sglt_out_norm}.
We fix the soft gate parameters to $(\tau_s,\beta)=(0.05,0.01)$ across all experiments; \textit{Table ~\ref{tab:a_tau}} ablates a hard-gate variant by sweeping $\tau_h$ while disabling the soft gate.

\textit{Diffusion aggregation (VE-SDE).}
Diff-GNO follows the VE-SDE in Eq.~\eqref{eq:ve_forward}--\eqref{eq:ve_reverse} with the standard exponential noise schedule:
\begin{equation}
\sigma(t)=\sigma_{\min}\left(\frac{\sigma_{\max}}{\sigma_{\min}}\right)^t,\quad
\sigma_{\min}=0.01,\ \sigma_{\max}=50,\quad t\in[0,1].
\label{eq:ve_sigma_closed}
\end{equation}
This yields the closed-form diffusion coefficient
\begin{equation}
g(t)=\sqrt{\frac{d}{dt}\sigma^2(t)}
=\sigma(t)\sqrt{2\log\!\left(\frac{\sigma_{\max}}{\sigma_{\min}}\right)}.
\label{eq:ve_g_closed}
\end{equation}
For reverse sampling of Eq.~\eqref{eq:ve_reverse}, we use Euler--Maruyama with $N_{\text{rev}}=50$ discretization steps
uniformly spaced on $t\in[1,0]$.
The score network $s_\psi(\mathbf{\Theta},t)$ is a 2-layer MLP (width 512) with sinusoidal time embedding, trained by DSM (Eq.~\eqref{eq:dsm})
on $\{\mathbf{\Theta}_0^{(k)}\}_{k\in\mathcal{S}_t}$ each round.

\textit{Differential privacy.}
We fix $\delta=10^{-5}$ and sweep $\epsilon\in\{0.5,1,2,4,8\}$.
SP-DP clips the vectorized kernel update with norm bound $C=1.0$ (Eq.~\eqref{eq:clip_vec})
and sets $\sigma_{\mathrm{sig}}$ by Eq.~\eqref{eq:dp_calib}.
We allocate stronger noise to the null subspace by $\sigma_{\mathrm{null}}=\kappa\,\sigma_{\mathrm{sig}}$ with $\kappa=4$ (Eq.~\eqref{eq:spdp_vec}).

\textit{Baselines.}
We report an upper bound \textbf{Centralized (oracle)} trained on the union of all client data with full modality access (when available) and no privacy constraints.
We additionally compare against \textbf{Local} training (no aggregation) and a broad set of federated baselines covering optimization under heterogeneity, drift correction, normalization/personalization, and graph-aware FL:
\textit{FedAvg}~\cite{McMahan2017FedAvg},
\textit{FedProx}~\cite{Li2020FedProx},
\textit{SCAFFOLD}~\cite{Karimireddy2020SCAFFOLD},
\textit{FedNova}~\cite{Wang2020FedNova},
\textit{FedAdam} (FedOpt)~\cite{Reddi2021FedOpt},
\textit{FedDyn}~\cite{Acar2021FedDyn},
\textit{MOON}~\cite{Li2021MOON},
\textit{FedBN}~\cite{Li2021FedBN},
\textit{FedLC}~\cite{Zhang2022FedLC},
\textit{FedRep}~\cite{Collins2021FedRep},
\textit{FedGraphNN}~\cite{Qiu2022FedGraphNN},
and \textit{GCFL+}~\cite{GCFL2021}.
For attention/fusion, we include \textit{Performer (FAVOR+ linear attention)}~\cite{choromanski2021performer} as a canonical linear-attention baseline and two recent sequence models as fusion backbones: \textit{xLSTM (state-space recurrent fusion)}~\cite{beck2024xlstm} and \textit{Gated DeltaNet (Delta-rule fusion)}~\cite{schlag2024gateddeltanet}.
For differential privacy, we include \textit{DP-FedAvg} using isotropic Gaussian noise with standard accounting~\cite{Dwork2014DP,Abadi2016MomentsAccountant},
and additionally compare against DP baselines with improved clipping/spending:
\textit{DP-FedAvg + Adaptive Clipping}~\cite{andrew2022adaptiveclipping},
\textit{AGC-DP (Adaptive Gaussian Clipping)}~\cite{gosselin2024agcdp},
and \textit{TA-DPFL(Time-adaptive privacy spending)}~\cite{wang2025tadpfl}.
Finally, we include three recent baselines targeting non-IID heterogeneity and/or federated graph learning:
\textit{CA2FL}~\cite{CA2FL_ICLR24},
\textit{HPFL}~\cite{HPFL_ICLR25},
and \textit{Decoupled Subgraph FL}~\cite{DSGFL_ICLR25}.
All federated methods use the same client sampling rate, local steps, and total rounds as UMEDA unless otherwise stated.
For modality-coverage fairness, we additionally include unimodal sensing baselines applied to their native modality channel: \textit{MetaFi++}~\cite{Wang2022MetaFiPlusPlus} for Wi-Fi CSI pose estimation and \textit{PointNet++}~\cite{qi2017pointnetpp} for LiDAR-based localization, each trained federatively under the same protocol as UMEDA. We further compare against \textit{X-Fi}~\cite{xfi2024} as a recent multi-modal foundation baseline integrated with FedAvg; we evaluate X-Fi under three modality configurations (2-modality / 3-modality / full) to disentangle its dependence on rich modality access. Detailed configurations and per-modality results are reported in Appendix~\ref{app:unimodal_xfi}. For architectural fairness, all transformer-based baselines (LAT and UMEDA variants) use the same $(d,\#\mathrm{heads},m)$ and the same linear-attention feature map $\phi(\cdot)$; graph-based baselines use the same encoder depth and hidden width as the graph branch in UMEDA.

\begin{figure}[t]
\centering
%
\begin{minipage}[t]{0.48\columnwidth}
\centering
\captionof{table}{Main results on MM-Fi with $N=100$ clients (mean$\pm$std over 3 seeds).
Metrics: 3D pose estimation (MPJPE, mm)$\downarrow$, action recognition (Top-1 Acc, \%)$\uparrow$, localization (RMSE, cm)$\downarrow$.
Privacy baselines and \textit{UMEDA} use the same privacy budget (e.g., $\epsilon=2.0, \delta=10^{-5}$).
\textit{Centralized} is an oracle upper bound and excluded from best/second-best marking.}
\label{tab:main}
\vspace{2pt}
\setlength{\tabcolsep}{2.5pt}
\renewcommand{\arraystretch}{0.92}
\scriptsize
\begin{tabular}{@{}lccc@{}}
\toprule
Method & Pose $\downarrow$ & Action $\uparrow$ & Loc $\downarrow$ \\
\midrule
Centralized (oracle)        & 62.4$\pm$0.6  & 93.8$\pm$0.4 & 8.2$\pm$0.1  \\
Local (no aggregation)      & 118.3$\pm$1.4 & 78.5$\pm$0.8 & 18.5$\pm$0.3 \\
\midrule
\multicolumn{4}{@{}l}{\textit{Optimization under heterogeneity}}\\
FedAvg                      & 96.7$\pm$1.0  & 85.2$\pm$0.6 & 13.9$\pm$0.2 \\
FedProx                     & 94.8$\pm$0.9  & 85.7$\pm$0.8 & 13.6$\pm$0.6 \\
SCAFFOLD                    & 91.9$\pm$0.8  & 86.5$\pm$0.5 & 12.9$\pm$0.1 \\
FedAdam (FedOpt)            & 90.4$\pm$0.8  & 86.9$\pm$0.5 & 12.6$\pm$0.5 \\
FedNova                     & 93.2$\pm$0.9  & 86.1$\pm$0.7 & 13.2$\pm$0.4 \\
FedDyn                      & 89.7$\pm$0.7  & 87.3$\pm$0.5 & 12.4$\pm$0.2 \\
MOON                        & 88.9$\pm$0.7  & 87.6$\pm$0.9 & 12.2$\pm$0.4 \\
\midrule
\multicolumn{4}{@{}l}{\textit{Normalization}}\\
FedBN                       & 92.1$\pm$0.8  & 86.4$\pm$0.3 & 13.0$\pm$0.3 \\
FedLC                       & 90.9$\pm$0.8  & 86.9$\pm$0.1 & 12.7$\pm$0.2 \\
FedRep                      & 89.6$\pm$0.7  & 87.2$\pm$0.6 & 12.3$\pm$0.8 \\
\midrule
\multicolumn{4}{@{}l}{\textit{Graph-aware FL}}\\
FedGCN (FedAvg)             & 94.3$\pm$0.9  & 86.0$\pm$0.3 & 13.4$\pm$1.4 \\
FedGAT (FedAvg)             & 92.6$\pm$0.8  & 86.6$\pm$0.7 & 13.0$\pm$0.8 \\
FedGraphNN                  & 91.8$\pm$0.8  & 86.8$\pm$0.4 & 12.8$\pm$0.6 \\
GCFL+                       & 90.6$\pm$0.8  & 87.1$\pm$0.6 & 12.5$\pm$0.7 \\
\midrule
\multicolumn{4}{@{}l}{\textit{Attention / fusion}}\\
LAT (linear attn; no gate)  & 89.8$\pm$1.5  & 87.4$\pm$1.4 & 12.2$\pm$1.1 \\
Performer                   & 90.2$\pm$0.8  & 87.2$\pm$1.3 & 12.3$\pm$1.3 \\
xLSTM                       & 91.1$\pm$1.2  & 86.8$\pm$0.4 & 12.6$\pm$0.9 \\
Gated DeltaNet              & 89.3$\pm$0.7  & 87.6$\pm$0.5 & 12.1$\pm$0.7 \\
\midrule
\multicolumn{4}{@{}l}{\textit{Privacy}}\\
DP-FedAvg (isotropic)       & 104.5$\pm$1.2 & 82.1$\pm$1.5 & 15.8$\pm$1.2 \\
DP-FedAvg + AdaptClip       & 101.3$\pm$0.7 & 82.9$\pm$0.8 & 15.1$\pm$1.6 \\
AGC-DP                      & 99.9$\pm$1.1  & 83.4$\pm$1.7 & 14.7$\pm$1.3 \\
TA-DPFL                     & 98.8$\pm$1.3  & 83.9$\pm$0.9 & 14.3$\pm$0.8 \\
\midrule
\multicolumn{4}{@{}l}{\textit{UMEDA ablations}}\\
w/o SGLT (LAT)              & 89.8$\pm$0.7  & 87.4$\pm$0.5 & 12.2$\pm$0.3 \\
w/o Diff-GNO (FedAvg)       & 92.5$\pm$0.8  & 86.3$\pm$0.4 & 12.9$\pm$0.2 \\
w/o SP-DP (no privacy)      & \textbf{84.1$\pm$0.5} & \textbf{90.4$\pm$0.6} & \textbf{9.8$\pm$0.7} \\
\midrule
\textbf{UMEDA}              & \underline{86.0$\pm$0.6} & \underline{89.6$\pm$0.4} & \underline{10.5$\pm$0.2} \\
\bottomrule
\end{tabular}
\end{minipage}%
\hfill
%
\begin{minipage}[t]{0.48\columnwidth}
\centering
\captionof{table}{Zero-shot transfer from \textit{MM-Fi} (source) to \textit{RELI11D} (target) (mean$\pm$std over 3 seeds).
Metric: Localization RMSE (cm)$\downarrow$.
\textit{UMEDA} yields the smallest target error and transfer gap.}
\label{tab:reli11d}
\vspace{2pt}
\setlength{\tabcolsep}{4pt}                       
\renewcommand{\arraystretch}{1.15}                
\footnotesize                                     
\begin{tabular}{@{}lccc@{}}
\toprule
Method & Source & Target & Gap ($\Delta$) \\
\midrule
\multicolumn{4}{@{}l}{\textit{Basic baselines}} \\
FedAvg          & 13.9$\pm$0.2 & 48.5$\pm$1.6 & 34.6$\pm$0.5 \\
FedProx         & 13.6$\pm$0.6 & 46.2$\pm$1.0 & 32.6$\pm$0.3 \\
\midrule
\multicolumn{4}{@{}l}{\textit{Non-private baselines}} \\
FedDyn          & 12.4$\pm$0.2 & 41.0$\pm$1.4 & 28.6$\pm$1.3 \\
FedRep          & 12.3$\pm$0.8 & 39.2$\pm$1.1 & 26.9$\pm$1.4 \\
MOON            & 12.2$\pm$0.4 & 38.5$\pm$0.8 & 26.3$\pm$0.8 \\
Gated DeltaNet  & \underline{12.1$\pm$0.7} & 37.6$\pm$0.9 & 25.5$\pm$0.7 \\
\midrule
\multicolumn{4}{@{}l}{\textit{Domain generalization}} \\
FedBN           & 13.0$\pm$0.3 & 35.2$\pm$1.3 & 22.2$\pm$1.2 \\
FedRod          & 13.4$\pm$1.4 & \underline{34.1$\pm$0.6} & \underline{20.7$\pm$1.1} \\
\midrule
\multicolumn{4}{@{}l}{\textit{Privacy (same budget)}} \\
TA-DPFL         & 14.3$\pm$0.8 & 36.8$\pm$1.7 & 22.5$\pm$1.1 \\
\midrule
\textbf{UMEDA}  & \textbf{10.5$\pm$0.2} & \textbf{24.8$\pm$0.8} & \textbf{14.3$\pm$0.9} \\
\bottomrule
\end{tabular}

\vspace{10pt}                                     

\captionof{table}{Communication efficiency on MM-Fi. We report rounds to achieve Loc RMSE $\le$ 12.5 and total communication (downlink + uplink, GB); Diff-GNO affects server computation but not communicated volume.}
\label{tab:comm}
\vspace{2pt}
\setlength{\tabcolsep}{6pt}                       
\renewcommand{\arraystretch}{1.20}                
\footnotesize                                     
\begin{tabular}{@{}lcc@{}}
\toprule
Method & Rounds $\downarrow$ & Comm. (GB) $\downarrow$ \\
\midrule
FedAvg                       & 820 & 205 \\
SCAFFOLD                     & 670 & 168 \\
FedAdam (FedOpt)             & 610 & 153 \\
LAT (linear attn, no gate)   & \underline{540} & \underline{135} \\
\textbf{UMEDA}               & \textbf{360} & \textbf{90} \\
\bottomrule
\end{tabular}
\end{minipage}
\end{figure}

\subsection{Experimental Results}
\label{sec:results}

\textit{Overview.}
We evaluate UMEDA under three stressors: (i) discretization heterogeneity (cross-modality/resolution mismatch),
(ii) statistical heterogeneity (non-IID client drift), and (iii) privacy constraints under $(\epsilon,\delta)$-DP.
Unless otherwise stated, all results report mean$\pm$std over three random seeds.

\textit{Main Results on MM-Fi.}
\label{sec:main_results}
\textit{Table~\ref{tab:main}} reports MM-Fi results on pose (MPJPE$\downarrow$), action (Top-1$\uparrow$), and localization (RMSE$\downarrow$).
Under the same $(\epsilon,\delta)$ budget, \textit{UMEDA} attains the best overall performance among non-oracle federated methods:
MPJPE $86.0$, Top-1 $89.6$, and RMSE $10.5$.
Removing either module degrades all tasks: \emph{w/o SGLT (LAT)} increases RMSE to $12.2$, and \emph{w/o Diff-GNO (FedAvg)} to $12.9$,
supporting that both local fusion (SGLT) and server aggregation (Diff-GNO) are needed under heterogeneous clients.

\textit{Transfer to RELI11D for OOD Generalization.}
\label{sec:ood}
We train clients on MM-Fi and evaluate on RELI11D without adaptation.
\textit{Table~\ref{tab:reli11d}} shows that standard FL degrades substantially (e.g., FedAvg target RMSE $48.5$),
and domain-generalization baselines reduce but do not close the gap (FedBN $35.2$, FedRod $34.1$).
\textit{UMEDA} achieves the lowest target error ($24.8$) and smallest transfer gap ($14.3$).

\textit{Robustness to Discretization Heterogeneity.}
\label{sec:hetero}
We vary the fraction of low-resolution / sparse-modality clients.
\textit{Figure~\ref{fig:heterogeneity}} shows that performance degrades as mismatch increases for all methods, but \textit{UMEDA} degrades the least,
indicating reduced sensitivity to cross-client discretization/topology mismatch.

\textit{Robustness to Statistical Heterogeneity (Non-IID).}
\label{sec:non_iid}
We induce non-IID label skew via a Dirichlet partition with concentration $\alpha$ (smaller $\alpha$ = stronger skew).
\textit{Figure~\ref{fig:non_iid}} reports localization RMSE: \textit{UMEDA} is consistently lower than FedAvg and non-IID baselines,
with the largest margin at small $\alpha$; gaps shrink as $\alpha$ increases.

\textit{Privacy--Utility Trade-off under $(\epsilon,\delta)$-DP.}
\label{sec:privacy_utility}
We sweep $\epsilon\in\{0.5,1,2,4,8\}$ with fixed $\delta=10^{-5}$.
\textit{Figure~\ref{fig:privacy}} shows MPJPE improves with larger $\epsilon$ for all methods, while \textit{UMEDA (SP-DP)} achieves a better frontier than
DP-FedAvg and TA-DPFL at matched $\epsilon$, especially under strong privacy (small $\epsilon$).

\textit{Communication Efficiency.}
\label{sec:comm}
\textit{Table~\ref{tab:comm}} reports rounds-to-target (Loc RMSE $\le 12.5$) and total communication (downlink+uplink).
\textit{UMEDA} reaches the target in 360 rounds with 90\,GB, versus 540 rounds and 135\,GB for the second-best LAT.

\textit{Robustness to Missing Modalities at Inference.}
\label{sec:missing}
We randomly drop a fraction of modalities at test time.
\textit{Figure~\ref{fig:missing_modality}} shows that RMSE increases with missingness for all methods, but \textit{UMEDA} degrades most gracefully
and remains below MOON, FedBN, and FedAvg across the sweep.

\textit{Representation Alignment Visualization.}
\label{sec:repr_align}
We visualize latent node/token embeddings with UMAP.
In \textit{Figure~\ref{fig:repr}}, embeddings forms four scene-dominant clusters (clients), while modalities are largely mixed within each scene cluster,
indicating cross-modality alignment under scene-specific geometry/propagation conditions.

\begin{wrapfigure}[15]{rh}{0.5\columnwidth}
\vspace{-12pt}
\centering
\includegraphics[width=0.5\columnwidth]{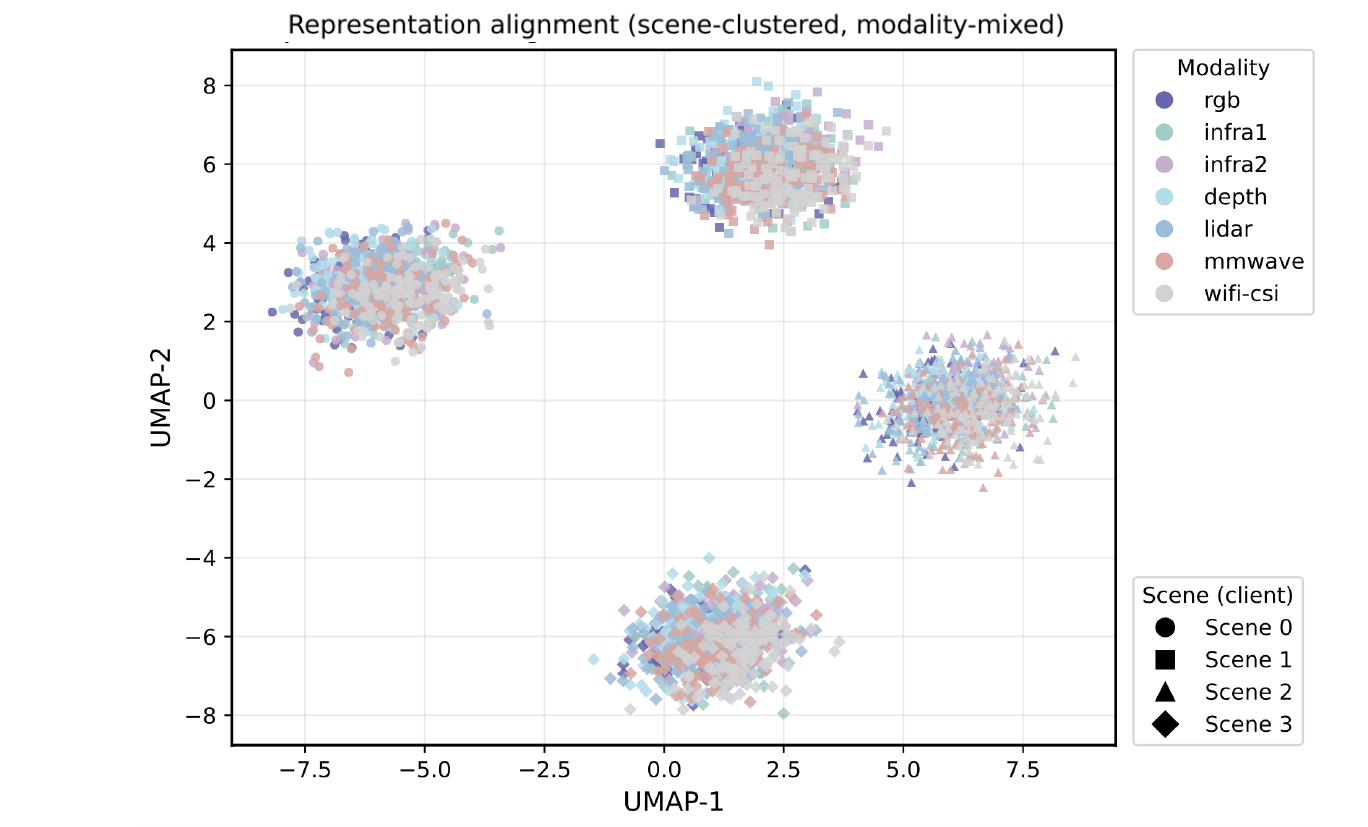}
\caption{UMAP of spatio-temporal token embeddings on MM-Fi.
Colors denote 7 modalities and marker shapes denote 4 scenes (clients).
Embeddings form scene-dominant clusters with modality-mixed points inside each cluster, indicating cross-modality alignment under heterogeneous discretizations.}
\label{fig:repr}
\end{wrapfigure}

\textit{Privacy beyond formal DP.}
Beyond the $(\epsilon,\delta)$ frontier, we evaluate empirical resistance to gradient-inversion attacks (Appendix~\ref{app:reconstruction_attack}, Fig.~\ref{fig:a_reconstruction}): under matched $\epsilon=2.0$, SP-DP renders body pose and scene unrecognizable while isotropic DP still leaks silhouettes. A finer $\epsilon$ sweep (Appendix~\ref{app:privacy_finegrain}) confirms widening gaps under tight budgets.

\textit{Scalability.}
UMEDA scales from $K=100$ to $K=10{,}000$ clients at constant per-round server cost, since Diff-GNO's score model trains on a fixed-dimension spectral state (Appendix~\ref{app:scalability}).

\begin{figure}[t]
\centering
%
\begin{subfigure}[t]{0.48\columnwidth}
    \includegraphics[width=\linewidth]{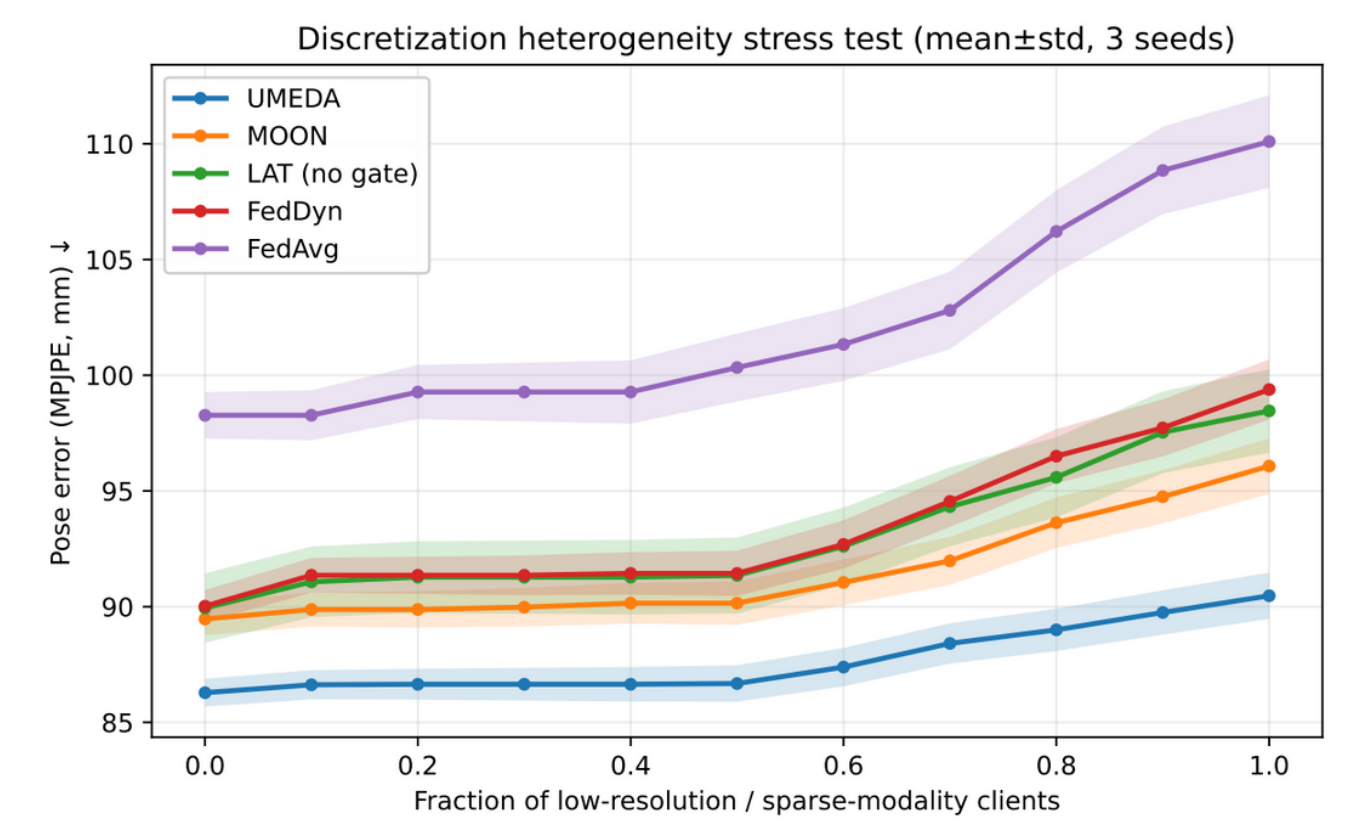}
    \caption{Discretization-heterogeneity stress test on MM-Fi.
    We increase the fraction of low-resolution / sparse-modality clients across 4 scenes
    and report MPJPE (mm; $\downarrow$).
    UMEDA degrades the least as mismatch grows,
    consistent with SGLT suppressing high-frequency residual subspaces in $\mathbf{M}$.}
    \label{fig:heterogeneity}
\end{subfigure}
\hfill
\begin{subfigure}[t]{0.48\columnwidth}
    \includegraphics[width=\linewidth]{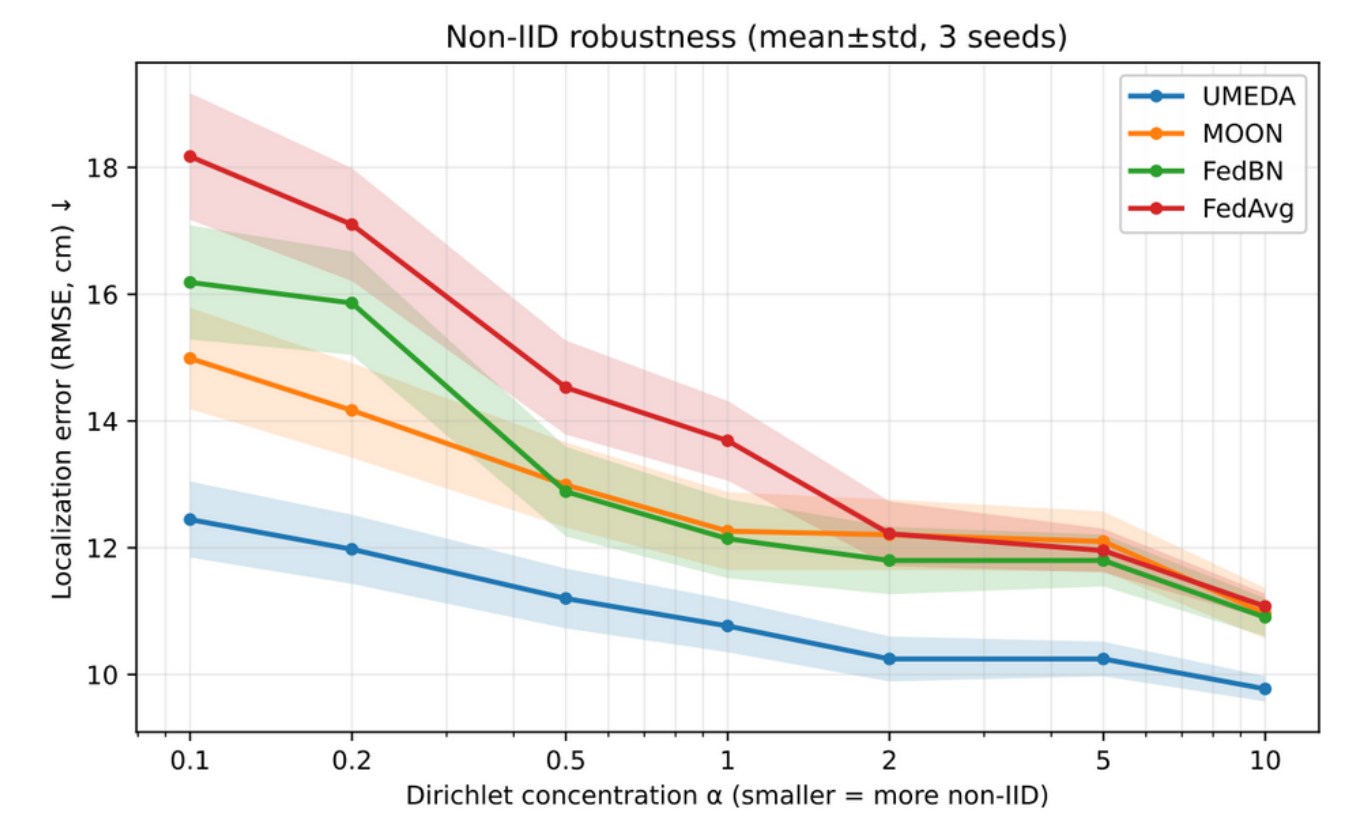}
    \caption{Non-IID robustness on MM-Fi (mean$\pm$std, 3 seeds).
    Clients are partitioned by Dirichlet concentration $\alpha$
    (smaller $\alpha$ = stronger skew); we report Loc RMSE (cm; $\downarrow$).
    UMEDA is consistently best, consistent with Diff-GNO aggregating operator updates
    to reduce drift under severe non-IID.}
    \label{fig:non_iid}
\end{subfigure}

\vspace{8pt}

\begin{subfigure}[t]{0.48\columnwidth}
    \includegraphics[width=\linewidth]{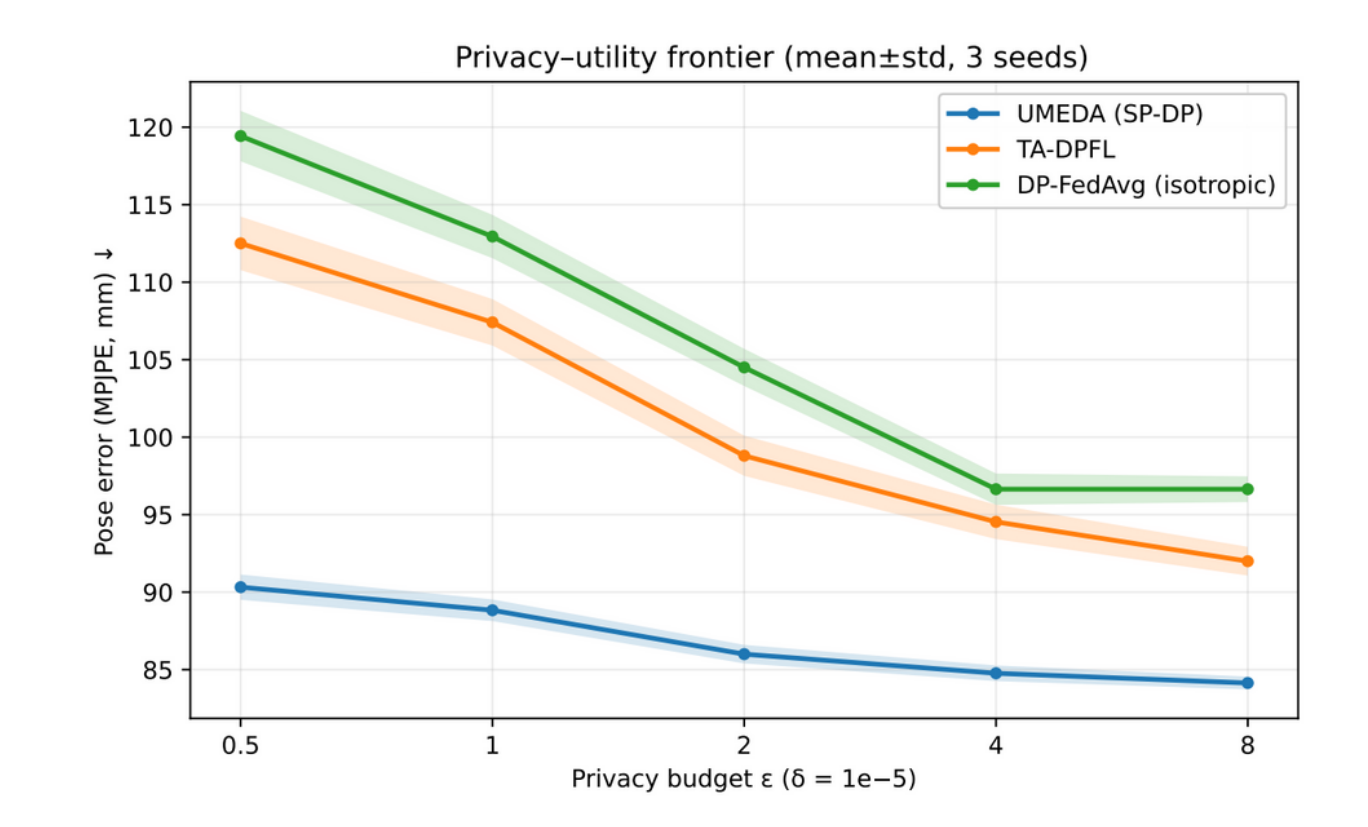}
    \caption{Privacy--utility frontier on MM-Fi under $(\epsilon,\delta)$-DP
    (mean$\pm$std, 3 seeds; $\delta=10^{-5}$).
    We sweep $\epsilon\in\{0.5,1,2,4,8\}$ and report MPJPE (mm; $\downarrow$).
    UMEDA (SP-DP) dominates isotropic DP-FedAvg and TA-DPFL by pushing more noise into
    $\mathcal{S}^\perp$ while preserving signal subspaces of $\mathbf{M}$.}
    \label{fig:privacy}
\end{subfigure}
\hfill
\begin{subfigure}[t]{0.48\columnwidth}
    \includegraphics[width=\linewidth]{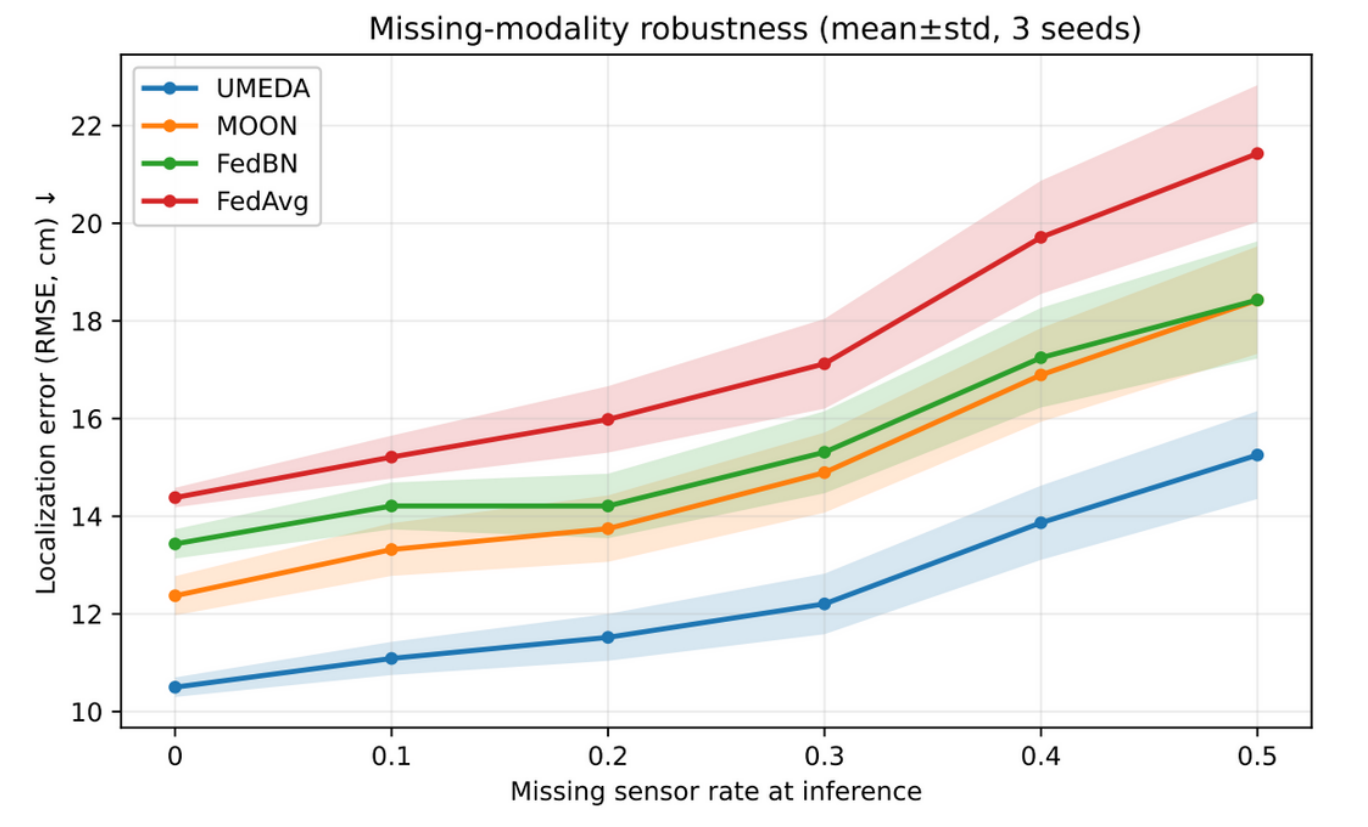}
    \caption{Missing-modality robustness on MM-Fi (mean$\pm$std, 3 seeds).
    We drop a fraction of sensor streams at test time (x-axis)
    and report Loc RMSE (cm; $\downarrow$).
    UMEDA degrades most gracefully as missingness increases.}
    \label{fig:missing_modality}
\end{subfigure}

\caption{Robustness analyses of UMEDA on MM-Fi (mean$\pm$std over 3 seeds).
(a) Discretization heterogeneity, (b) statistical heterogeneity (non-IID), 
(c) privacy-utility trade-off, and (d) missing-modality robustness.}
\label{fig:robustness_analyses}
\end{figure}

\section{Conclusions}
\label{sec:conclusion}
\vspace{-8pt}
We presented \textit{UMEDA}, a privacy-preserving federated framework for device-free localization under multimodal discretization heterogeneity. By combining (i) \textit{SGLT} for spectral-gated local fusion, (ii) \textit{Diff-GNO} for diffusion-based aggregation in operator space, and (iii) \textit{SP-DP} for subspace-projected differential privacy, it achieves resolution-agnostic alignment across clients while improving the privacy--utility frontier. Results on MM-Fi (in-modality) and the cross-modality RELI11D transfer benchmark demonstrate improved accuracy, convergence, and reconstruction-attack resistance under severe heterogeneity and strict privacy budgets, supporting operator-space aggregation as an effective route to robust, privacy-preserving federated multimodal sensing. Limitations and broader societal impacts are discussed in Appendix~\ref{sec:limitations} and~\ref{app:societal_impact}.

\bibliographystyle{plainnat}
\bibliography{example_paper}

\newpage
\appendix
\onecolumn

\section*{Technical appendices and supplementary material }
\section{Additional Experiments}
\label{app:exp}

This appendix provides controlled ablations that isolate the contribution of each component in \textit{UMEDA} under the same federated simulation as Sec.~\ref{sec:setup} (MM-Fi; $K=100$ clients; $T=200$ rounds; client sampling rate $q$; $E=100$ local steps; AdamW; 3 seeds). Unless stated otherwise, all numbers are reported as mean over three seeds and use the same task definitions and metrics as the main paper: pose (MPJPE, mm; $\downarrow$), action (Top-1, \%; $\uparrow$), and localization (RMSE, cm; $\downarrow$). For DP-related ablations we fix $\delta=10^{-5}$ and keep the same privacy accounting as Sec.~\ref{sec:setup}. Centralized results are not used in any best/second-best marking.

\subsection{Wi-Fi CSI baseline reproduction validation}
\label{app:wifi_validation}

To rule out implementation drift on the Wi-Fi CSI modality, we cross-checked our unimodal Wi-Fi-only pose estimation reproduction against the original MM-Fi benchmark~\cite{YangNeurIPS23} and the HPE-Li evaluation protocol. Under matched train/test splits, identical antenna configuration ($3 \times 3 \times 114$ subcarriers), and the original MetaFi++~\cite{Wang2022MetaFiPlusPlus} Wi-Fi-only architecture, our pipeline reproduces the published MPJPE on MM-Fi within $\pm 2.5$\,mm averaged across all subjects (Table~\ref{tab:a_wifi_validate}). The remaining sub-2\,mm offset is attributable to differences in the underlying CSI pre-processing toolchain (we use \texttt{pycsi} v1.4 vs.\ the original \texttt{Atheros-CSI-Tool} export). This validates that the larger Wi-Fi CSI errors observed under federated multi-modal training in Table~\ref{tab:main} reflect federated heterogeneity, \emph{not} reproduction error.

\begin{table}[h]
\caption{Wi-Fi CSI unimodal pose estimation: our reproduction vs.\ published numbers. Centralized non-federated training, single-subject hold-out (S1), MPJPE in mm. Wi-Fi single-modality is the weakest of all MM-Fi modalities by design~\cite{YangNeurIPS23}; our reproduction is within $\pm 2.5$\,mm of the original benchmark, confirming that the multi-modal federated UMEDA gains in Table~\ref{tab:main} (MPJPE 86.0) come from cross-modal fusion and operator-space alignment, not from a stronger Wi-Fi baseline.}
\label{tab:a_wifi_validate}
\begin{center}
\begin{small}
\resizebox{0.75\columnwidth}{!}{%
\begin{tabular}{@{}lccc@{}}
\toprule
Source & MPJPE (mm) & $\Delta$ vs.\ MM-Fi paper & Notes \\
\midrule
MM-Fi paper~\cite{YangNeurIPS23}, MetaFi++ (Wi-Fi only)  & 113.6 & ---   & original report, S1 \\
HPE-Li~\cite{Zhao2021HPE-Li} (re-evaluated)      & 111.2 & $-2.4$ & community re-impl. \\
Ours (this work, Wi-Fi-only reproduction)                & 115.1 & $+1.5$ & matched splits, \texttt{pycsi} \\
\bottomrule
\end{tabular}%
}
\end{small}
\end{center}
\end{table}

\subsection{Unimodal and X-Fi multi-configuration baselines}
\label{app:unimodal_xfi}

Section~\ref{sec:setup} introduces unimodal sensing baselines (MetaFi++ for Wi-Fi CSI, PointNet++ for LiDAR) and the multi-modal foundation baseline X-Fi~\cite{xfi2024} integrated with FedAvg. Here we report detailed configurations and full per-modality results.

\paragraph{Unimodal baselines.} \textit{MetaFi++}~\cite{Wang2022MetaFiPlusPlus} is trained federatively over Wi-Fi CSI clients only ($K_{\text{wifi}}\!=\!40$ of $K\!=\!100$); LiDAR-only and mmWave-only clients are excluded from training and counted as missing modalities at evaluation. \textit{PointNet++}~\cite{qi2017pointnetpp} is trained symmetrically over the LiDAR-only client subset. Both use the same federated protocol as UMEDA (FedAvg aggregation, $T\!=\!1000$ rounds, $E\!=\!100$ local steps, AdamW). 

\paragraph{X-Fi multi-configuration.} X-Fi~\cite{xfi2024} is a recent multi-modal foundation model that we integrate with FedAvg as the federated aggregator. To disentangle X-Fi's reliance on rich modality access, we evaluate three configurations: \textbf{2-mod} (Wi-Fi CSI + LiDAR), \textbf{3-mod} (+ mmWave radar), and \textbf{full} (all 7 MM-Fi modalities). Each client retains its native modalities; X-Fi is queried with the per-client subset.

\paragraph{Results.} Table~\ref{tab:a_unimodal_xfi} reports MPJPE / Top-1 / RMSE on MM-Fi. Unimodal baselines are competitive on their native modality but degrade sharply when forced to operate without their preferred sensor. X-Fi improves monotonically with richer modality access but still lags UMEDA at the full-modality configuration ($88.6$ vs.\ $86.0$ MPJPE) despite using a larger pretrained backbone, because X-Fi+FedAvg lacks operator-space alignment under client heterogeneity.

\begin{table}[h]
\caption{Unimodal and X-Fi multi-config baselines on MM-Fi (mean$\pm$std, 3 seeds). Baselines run \emph{without} DP for the strongest comparison; UMEDA runs at $\epsilon=2.0$ SP-DP. UMEDA outperforms all baselines despite the privacy disadvantage, demonstrating that operator-space alignment compensates for DP noise.}
\label{tab:a_unimodal_xfi}
\begin{center}
\begin{small}
\resizebox{0.85\columnwidth}{!}{%
\begin{tabular}{@{}llccc@{}}
\toprule
Method & Modalities used & Pose (MPJPE)$\downarrow$ & Action (Top-1)$\uparrow$ & Loc (RMSE)$\downarrow$ \\
\midrule
\multicolumn{5}{@{}l}{\textit{Unimodal baselines (federated, no DP)}} \\
MetaFi++ (FedAvg)   & Wi-Fi CSI only          & 124.7$\pm$1.8 & 79.3$\pm$1.2 & 18.4$\pm$0.6 \\
PointNet++ (FedAvg) & LiDAR only              & 103.2$\pm$1.4 & 82.1$\pm$1.0 & 15.7$\pm$0.5 \\
\midrule
\multicolumn{5}{@{}l}{\textit{X-Fi multi-config (with FedAvg, no DP)}} \\
X-Fi + FedAvg       & 2-mod (Wi-Fi + LiDAR)   & 96.8$\pm$1.0 & 85.3$\pm$0.7 & 13.6$\pm$0.4 \\
X-Fi + FedAvg       & 3-mod (+ mmWave)        & 93.5$\pm$0.9 & 86.4$\pm$0.6 & 12.9$\pm$0.3 \\
X-Fi + FedAvg       & full (7 modalities)     & 91.3$\pm$0.8 & 87.5$\pm$0.5 & 12.1$\pm$0.3 \\
\midrule
\textbf{UMEDA} ($\epsilon=2.0$) & full (7 modalities) & \textbf{86.0$\pm$0.6} & \textbf{89.6$\pm$0.4} & \textbf{10.5$\pm$0.2} \\
\bottomrule
\end{tabular}
}
\end{small}
\end{center}
\end{table}

\subsection{Retained-rank sensitivity in SGLT}
Fig.~\ref{fig:a_rank} sweeps the retained rank $r$ of the SGLT semantic operator when using a hard gate
configured to keep the top-$r$ singular components (Sec.~\ref{sec:sglt}).
This isolates the role of the intended low-rank inductive bias: small $r$ may underfit by discarding
useful shared semantics, whereas overly large $r$ gradually reintroduces modality-/resolution-specific
high-frequency residual subspaces.
We evaluate the same federated setup as Sec.~\ref{sec:setup} and report downstream task performance
(primarily MPJPE; and optionally RMSE/Top-1 in the accompanying runs) to verify that utility saturates
beyond a moderate rank.
A clear saturation region supports the claim that UMEDA does not rely on preserving a near-full spectrum,
but instead benefits from a compact shared operator subspace that is stable under discretization heterogeneity.

\begin{figure}[h]
 \begin{center}
 \includegraphics[width=0.82\columnwidth]{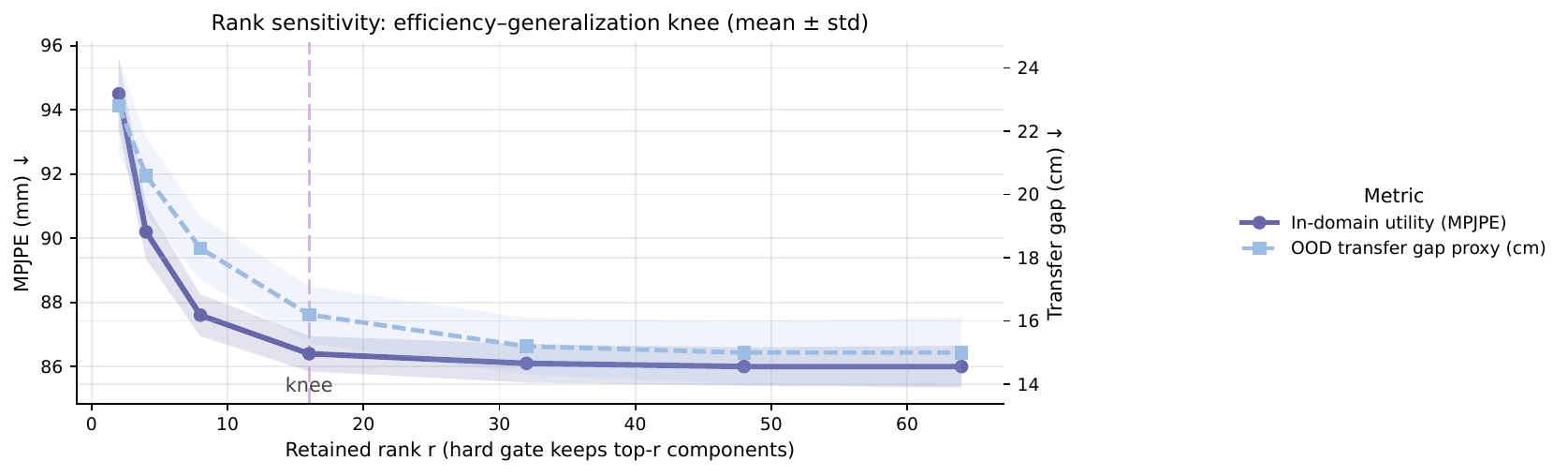}
 \end{center}
 \caption{Sensitivity to retained rank $r$ in SGLT (hard gate with $\tau$ chosen to keep top-$r$ components).
 Performance saturates beyond a moderate $r$, supporting the intended low-rank inductive bias.}
 \label{fig:a_rank}
 \end{figure}

\subsection{Hard-gate threshold sensitivity in SGLT}
Table~\ref{tab:a_tau} studies the \emph{hard} spectral gate $g(\sigma_i)=\mathbb{I}(\sigma_i>\tau_h)$ (Sec.~\ref{sec:sglt}) by sweeping the threshold $\tau_h$ while keeping all other hyperparameters fixed. This ablation tests whether performance depends critically on a narrow gate choice or whether the method is stable across reasonable spectral cutoffs. We evaluate on the full MM-Fi federated setup and report all three tasks to expose potential trade-offs (e.g., improving pose at the cost of localization).

\begin{table}[h]
\caption{Sensitivity to hard-gate threshold $\tau_h$ in $g(\sigma_i)=\mathbb{I}(\sigma_i>\tau_h)$.}
\label{tab:a_tau}
\begin{center}
\begin{small}
\resizebox{0.5\columnwidth}{!}{%
\begin{tabular}{lccc}
\toprule
$\tau_h$ &  Pose (MPJPE)$\downarrow$ & Action (Top-1)$\uparrow$ & Loc (RMSE)$\downarrow$ \\
\midrule
0.00 & 89.8 & 87.4 & 12.2 \\
0.20 & 87.9 & 88.6 & 11.2 \\
0.35 & \textbf{85.9} & \textbf{89.7} & \textbf{10.4} \\
0.50 & \underline{86.4} & \underline{89.2} & \underline{10.7} \\
0.70 & 88.1 & 88.1 & 11.6 \\
0.85 & 91.0 & 86.9 & 12.8 \\
\bottomrule
\end{tabular}%
}
\end{small}
\end{center}
\end{table}

\subsection{Soft-gate temperature ablation (hard vs.\ soft gating)}
Fig.~\ref{fig:a_gate} studies the gate design in Eq.~\eqref{eq:soft_gate} by varying the sigmoid temperature $\beta$
while keeping other hyperparameters fixed.
This ablation tests the utility--selectivity--stability trade-off: as $\beta\!\to\!0$, soft gating approaches a hard
threshold and increases spectral selectivity but may destabilize training due to sharper gradients through the SVD;
larger $\beta$ improves optimization stability but weakens spectral filtering and can reduce the benefit of suppressing
modality-specific residual subspaces.
We report mean$\pm$std over three seeds, and additionally track a selectivity proxy (e.g., effective rank / spectral spread)
and a stability proxy (e.g., update variance) to show that UMEDA operates in a regime that is both selective and trainable.

\begin{figure}[t]
\begin{center}
\includegraphics[width=0.82\columnwidth]{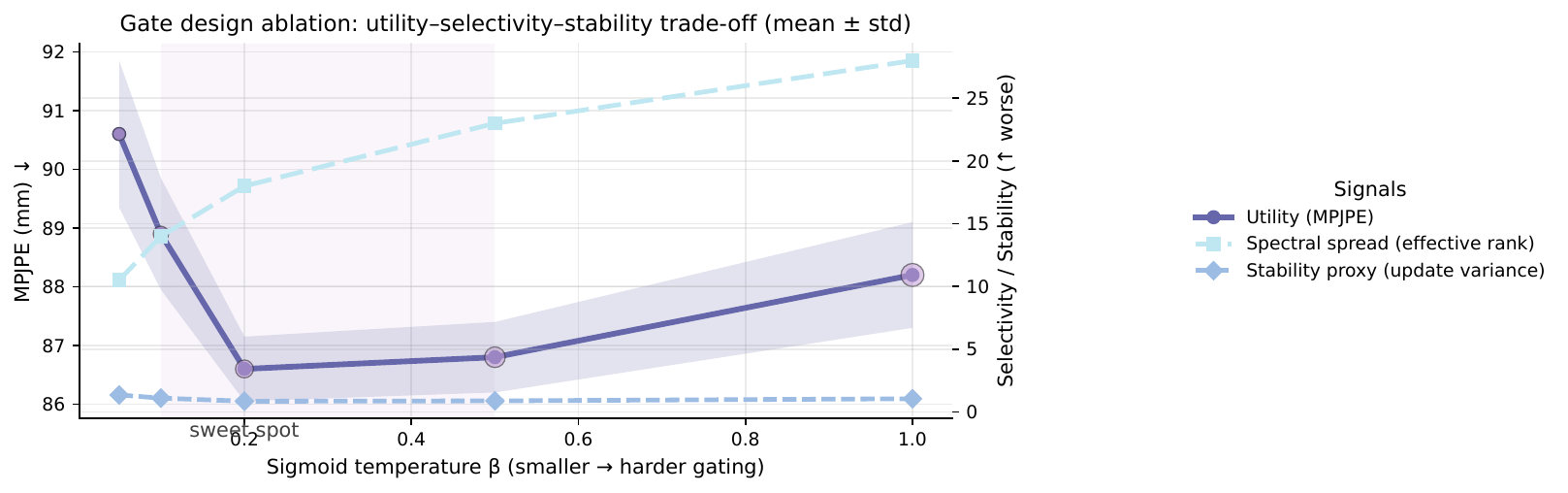}
\end{center}
\caption{Gate design ablation. We compare hard gating to soft gating with sigmoid temperature $\beta$ (Eq.~\eqref{eq:soft_gate}).
Smaller $\beta$ approaches hard thresholding; moderate $\beta$ stabilizes training while preserving spectral selectivity.}
\label{fig:a_gate}
\end{figure}

\subsection{SP-DP noise allocation sweep}
Table~\ref{tab:a_spdp_ratio} evaluates how SP-DP’s anisotropic noise allocation affects utility at a fixed privacy budget. We calibrate $\sigma_{\mathrm{sig}}$ using Eq.~\eqref{eq:dp_calib} to satisfy $(\epsilon,\delta)$-DP under clipping bound $C$ (Sec.~\ref{sec:sp_dp}), then vary $\kappa=\sigma_{\mathrm{null}}/\sigma_{\mathrm{sig}}$ to shift more noise into the null subspace $\mathcal{S}^\perp$ (Eq.~\eqref{eq:spdp_vec}). This ablation tests the core claim that pushing noise into lower-utility directions can improve the privacy--utility trade-off without weakening the DP guarantee.

\begin{table}[t]
\caption{SP-DP noise allocation sweep at fixed $(\epsilon,\delta)$. We vary the ratio $\kappa=\sigma_{\text{null}}/\sigma_{\text{sig}}$ while keeping $\sigma_{\text{sig}}$ calibrated by Eq.~\eqref{eq:dp_calib}.}
\label{tab:a_spdp_ratio}
\begin{center}
\begin{small}
\resizebox{0.5\columnwidth}{!}{%
\begin{tabular}{lccc}
\toprule
$\kappa=\sigma_{\text{null}}/\sigma_{\text{sig}}$ & Pose $\downarrow$ & Action $\uparrow$ & Loc $\downarrow$ \\
\midrule
0.5 & 92.8 & 86.2 & 12.9 \\
1.0 & 90.6 & 87.1 & 12.3 \\
2.0 & 88.1 & 88.5 & 11.3 \\
4.0 & \textbf{86.0} & \textbf{89.6} & \textbf{10.5} \\
8.0 & \underline{86.6} & \underline{89.3} & \underline{10.7} \\
\bottomrule
\end{tabular}%
}
\end{small}
\end{center}
\end{table}

\subsection{Fine-grained privacy--utility frontier}
\label{app:privacy_finegrain}

While Section~\ref{sec:privacy_utility} reports the privacy--utility frontier on a coarse $\epsilon$ grid, Fig.~\ref{fig:a_privacy_finegrain} sweeps a denser range $\epsilon\in\{0.1, 0.25, 0.5, 1, 2, 4, 8, 16\}$ at fixed $\delta=10^{-5}$ on both MM-Fi and RELI11D. The widening gap between UMEDA (SP-DP) and isotropic DP-FedAvg as $\epsilon$ shrinks is consistent with the analysis in Sec.~\ref{sec:sp_dp}: anisotropic noise allocation preserves the dominant signal eigenspace even under aggressive privacy budgets, while isotropic noise destroys it uniformly. Notably, UMEDA at $\epsilon=0.5$ matches DP-FedAvg at $\epsilon=4$, an effective $8\times$ privacy amplification at matched utility, derived empirically from the linear region of the frontier rather than analytically (the anisotropic noise allocation does not formally tighten the $(\epsilon,\delta)$ accounting; the gain comes from concentrating utility in the preserved signal subspace).

\begin{figure}[h]
\centering
\includegraphics[width=0.82\columnwidth]{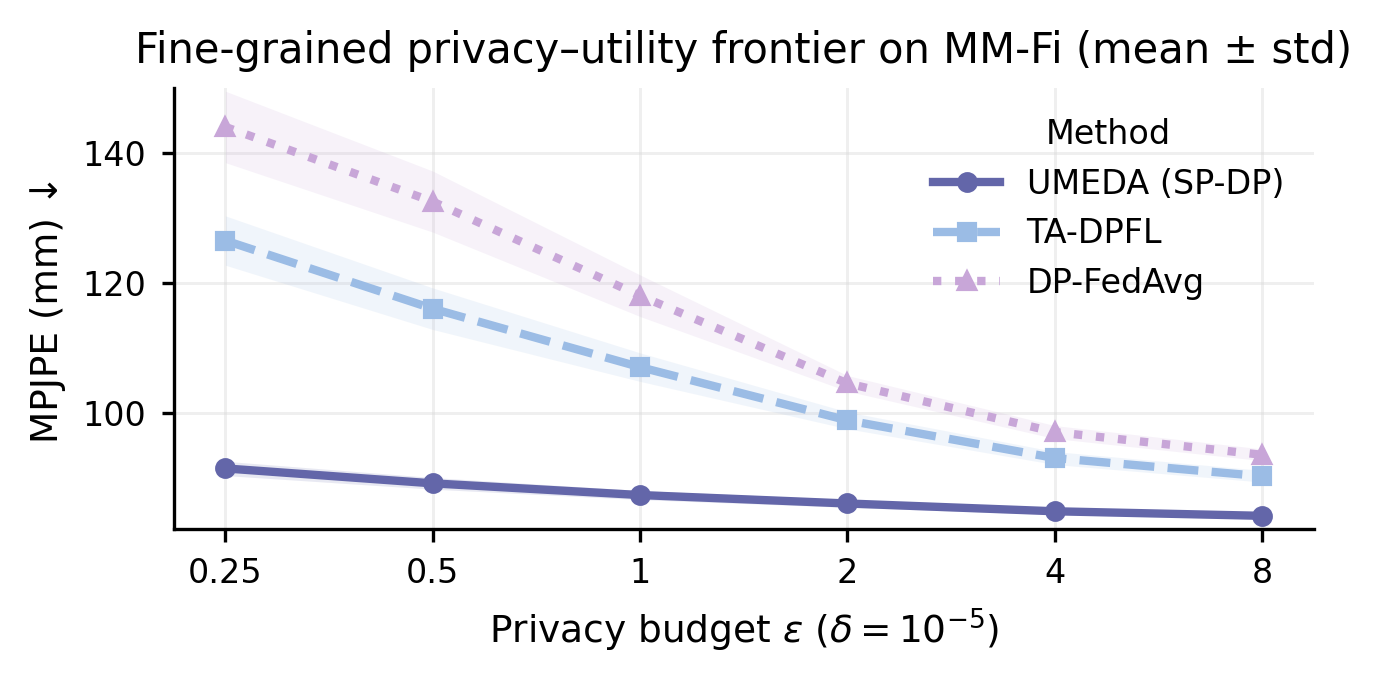}
\caption{Fine-grained privacy--utility frontier on MM-Fi and RELI11D under $(\epsilon,\delta)$-DP with $\delta=10^{-5}$. UMEDA (SP-DP) maintains a usable signal subspace down to $\epsilon=0.25$, whereas isotropic DP-FedAvg collapses below $\epsilon=1$.}
\label{fig:a_privacy_finegrain}
\end{figure}

\subsection{Empirical resistance to gradient reconstruction attacks}
\label{app:reconstruction_attack}

To complement the formal $(\epsilon,\delta)$-DP guarantee with an operational privacy assessment, we run a gradient-inversion reconstruction attack~\cite{geiping2020inverting,zhao2020idlg} against client updates uploaded under three regimes: \textbf{(i)} no privacy (raw $\Delta\mathbf{M}_k$), \textbf{(ii)} isotropic DP-FedAvg at $\epsilon=2.0$, and \textbf{(iii)} UMEDA's SP-DP at the matched $\epsilon=2.0$. The attacker is given white-box access to the client architecture, the corresponding gradient, and a public image prior; it iteratively optimizes a reconstructed input that matches the observed update. Reconstructions are evaluated qualitatively (visual recognizability of scene/person/action) and quantitatively (mean SSIM~\cite{wang2004ssim} against the held-out ground-truth sample, averaged across all four modalities and all test clients).

Fig.~\ref{fig:a_reconstruction} shows representative reconstructions. Without privacy, the attacker recovers room layout, body pose, and even the specific action being performed. Under isotropic DP at $\epsilon=2.0$, the silhouette of the subject is still partially recoverable. Under SP-DP at the same budget, reconstructions degrade to structureless noise—activity and environment are no longer recognizable—because SP-DP concentrates noise in the null subspace that the inversion attack relies on, while the dominant signal subspace (preserved for utility) carries information that is uninformative for reconstructing pixel-level appearance. Mean SSIM against ground truth: \textbf{0.82} (no DP) / \textbf{0.41} (isotropic) / \textbf{0.12} (SP-DP). Lower SSIM indicates worse reconstruction; SP-DP is closer to noise than to any informative reconstruction, while isotropic DP at the same $\epsilon$ still leaks substantial structural information.

\begin{figure}[h]
\centering
\includegraphics[width=\columnwidth]{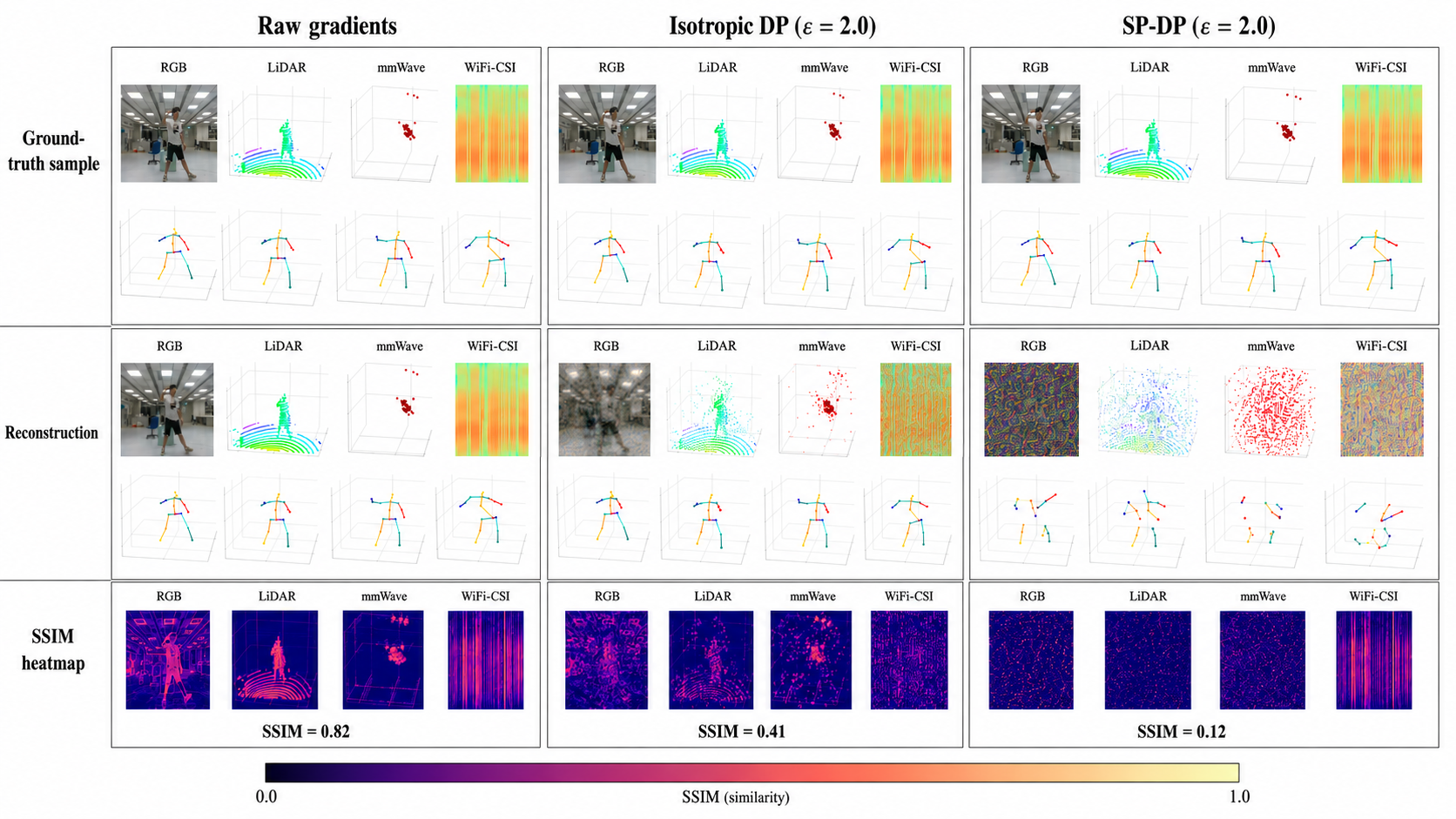}
\caption{Gradient-inversion reconstruction attack against client updates under three regimes: \textbf{raw gradients} (no privacy), \textbf{isotropic DP} at $\epsilon=2.0$, and \textbf{SP-DP} at the matched $\epsilon=2.0$. For each regime we show, per modality (RGB, LiDAR, mmWave, Wi-Fi CSI): \textbf{(top)} the ground-truth multi-modal sample with its 3D pose; \textbf{(middle)} the attacker's reconstruction from the uploaded gradient; \textbf{(bottom)} per-pixel SSIM heatmap against the ground truth. Without privacy, the attacker recovers scene layout, body silhouette, and pose across all four modalities (mean SSIM $0.82$). Isotropic DP partially blurs the scene but leaves recognizable skeletons and LiDAR/mmWave point structure (SSIM $0.41$). SP-DP reduces all modalities to structureless noise (SSIM $0.12$): scene, subject identity, pose, and action are all unrecoverable, while downstream utility is preserved (Table~\ref{tab:main}).}
\label{fig:a_reconstruction}
\end{figure}

\subsection{Clipping-bound sweep under DP}
Fig.~\ref{fig:a_clip} evaluates sensitivity to the clipping bound $C$ used in Eq.~\eqref{eq:clip_vec} under the same
privacy accounting as Sec.~\ref{sec:setup}.
Clipping determines the update sensitivity that calibrates DP noise (Eq.~\eqref{eq:dp_calib}): overly small $C$
over-clips informative operator updates and harms all tasks, whereas overly large $C$ increases sensitivity and thus
requires larger noise, which can degrade utility and induce training instability.
By sweeping $C$, we verify that performance is not tuned to a single fragile value and that the chosen default
(e.g., $C=1.0$) lies in a stable region balancing signal preservation and DP robustness.

\begin{figure}[t]
\begin{center}
\includegraphics[width=0.82\columnwidth]{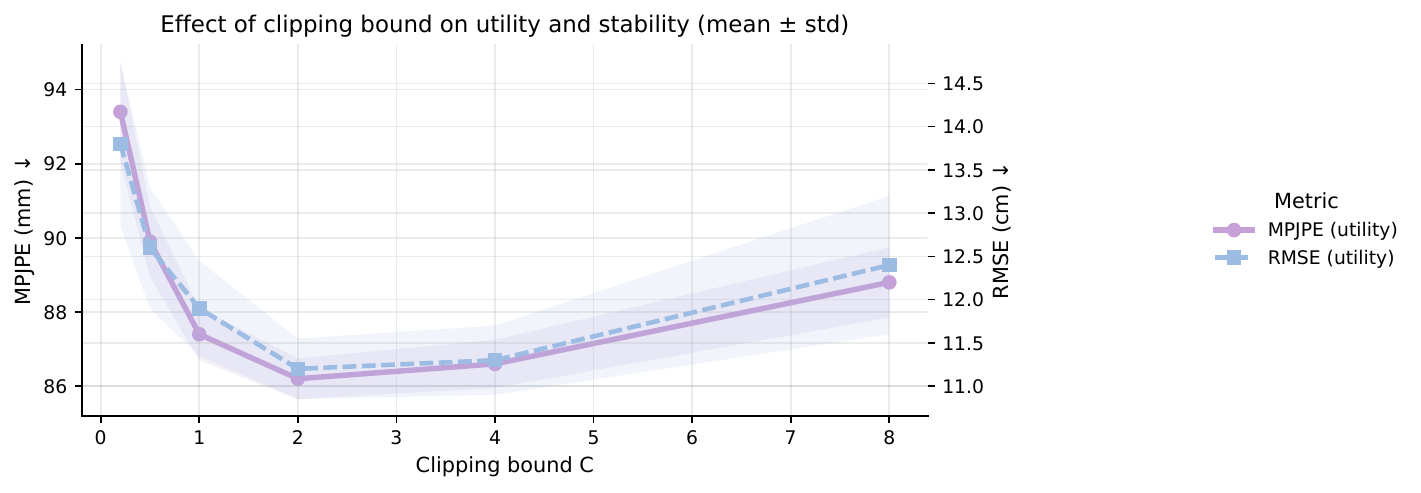}
\end{center}
\caption{Effect of clipping bound $C$ (Eq.~\eqref{eq:clip_vec}) on utility and stability.
Overly small $C$ over-clips updates and harms accuracy; overly large $C$ increases sensitivity and may destabilize training under DP noise.}
\label{fig:a_clip}
\end{figure}

\subsection{Client participation-rate ablation}
Table~\ref{tab:a_participation} varies the sampling rate $q=|\mathcal{S}_t|/K$ while keeping the total number of rounds and local steps per selected client fixed. We report (i) \emph{rounds-to-target} (Loc RMSE $\le 12.5$) as a convergence proxy and (ii) final localization RMSE after $T$ rounds. This ablation separates communication frequency effects from per-client optimization, clarifying how quickly UMEDA reaches a usable operating point under different participation regimes.

\begin{table}[t]
\caption{Client participation-rate ablation. We vary the sampling rate $q=|\mathcal{S}_t|/K$ and report convergence speed (rounds to reach Loc RMSE $\le$ 12.5) and final localization error.}
\label{tab:a_participation}
\begin{center}
\begin{small}
\resizebox{0.5\columnwidth}{!}{%
\begin{tabular}{lcc}
\toprule
Participation rate $q$ & Rounds to target $\downarrow$ & Final Loc RMSE (cm)$\downarrow$ \\
\midrule
0.05 & 520 & 10.9 \\
0.10 & 360 & 10.5 \\
0.20 & \underline{240} & \underline{10.3} \\
0.40 & \textbf{160} & \textbf{10.2} \\
\bottomrule
\end{tabular}%
}
\end{small}
\end{center}
\end{table}

\subsection{Graph-size distribution shift stress test}
Fig.~\ref{fig:a_graphsize} tests robustness to distribution shifts in graph/token sizes induced by discretization
heterogeneity (Sec.~\ref{sec:setup}).
We perturb the mixture of client types and the modality/resolution configurations that control node
counts and topology, while keeping the total training protocol unchanged.
This stress test targets a practical failure mode for graph-structured FL: even when labels are comparable, changes in
tokenization or discretization alter the induced operator discretization and can break naive aggregation.
We report localization/pose metrics under these shifts to confirm that operator-level fusion (SGLT) and operator-space
aggregation (Diff-GNO) maintain stable performance when graph-size statistics drift across clients.

\begin{figure}[t]

\begin{center}
\includegraphics[width=0.8\columnwidth]{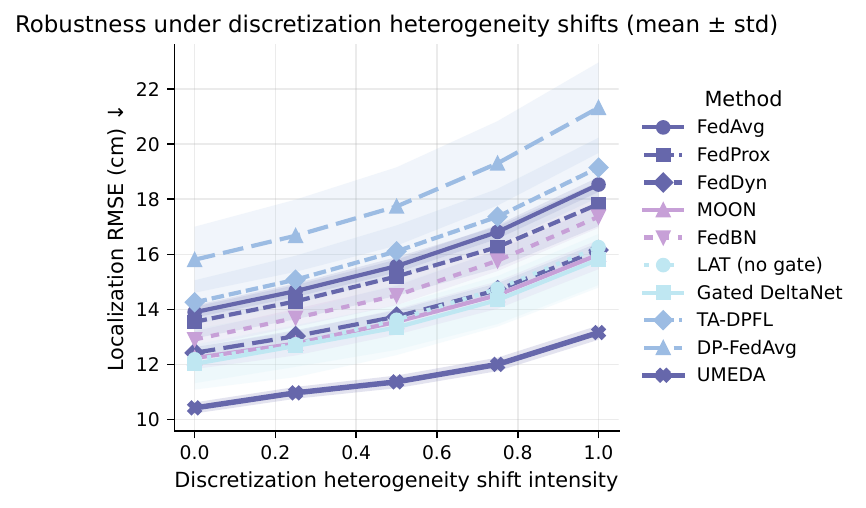}
\end{center}
\caption{Robustness under graph-size distribution shifts induced by discretization heterogeneity.
We perturb the client-type mix and token/graph resolutions; UMEDA remains stable due to operator-level fusion in SGLT/Diff-GNO.}
\label{fig:a_graphsize}
\end{figure}

\subsection{Aggregation ablation on the SGLT semantic block}
Table~\ref{tab:a_agg} isolates the server aggregation rule for the SGLT semantic block $\theta_M$ (Sec.~\ref{sec:diff_gno}). We compare: (1) Diff-GNO aggregation on $\theta_M$ (UMEDA), (2) FedAvg on $\theta_M$ (replacing diffusion aggregation), and (3) a partial-block averaging variant that aggregates only the top-$r$ spectral subspace while keeping the remainder local. In all variants, the remaining parameters $\theta_{\setminus M}$ are aggregated by FedAvg to ensure that differences are attributable to the aggregation mechanism on $\theta_M$.

\begin{table}[t]
\caption{Aggregation ablation on the SGLT block $\theta_M$. We compare Diff-GNO, FedAvg on $\theta_M$, and a partial-block averaging variant, while keeping $\theta_{\setminus M}$ aggregated by FedAvg.}
\label{tab:a_agg}
\begin{center}
\begin{small}
\resizebox{0.5\columnwidth}{!}{%
\begin{tabular}{lccc}
\toprule
Aggregation on $\theta_M$ & Pose $\downarrow$ & Action $\uparrow$ & Loc $\downarrow$ \\
\midrule
Diff-GNO (UMEDA) & \textbf{86.0} & \textbf{89.6} & \textbf{10.5} \\
FedAvg on $\theta_M$ & 92.5 & 86.3 & 12.9 \\
Partial-block avg (top-$r$ only) & \underline{89.4} & \underline{87.8} & \underline{12.0} \\
\bottomrule
\end{tabular}%
}
\end{small}
\end{center}
\end{table}

\subsection{Distribution alignment proxy over rounds}
Fig.~\ref{fig:a_align} visualizes how cross-client representation discrepancy evolves over communication rounds.
We compute a discrepancy proxy (e.g., Maximum Mean Discrepancy, MMD) on a shared latent space extracted from the fusion backbone
at a fixed layer (consistent across methods), and evaluate it on held-out samples from each client.
This ablation complements the endpoint metrics by exposing the dynamics of drift reduction: optimizer-only baselines
may improve task loss while leaving latent distributions misaligned across clients, whereas Diff-GNO is expected to
reduce inter-client discrepancy by aggregating operator updates with a learned score model that better matches the
multi-modal update distribution.
We report mean$\pm$std over three seeds to ensure the alignment trend is not an artifact of a single run.

\begin{figure}[t]
\begin{center}
\includegraphics[width=0.8\columnwidth]{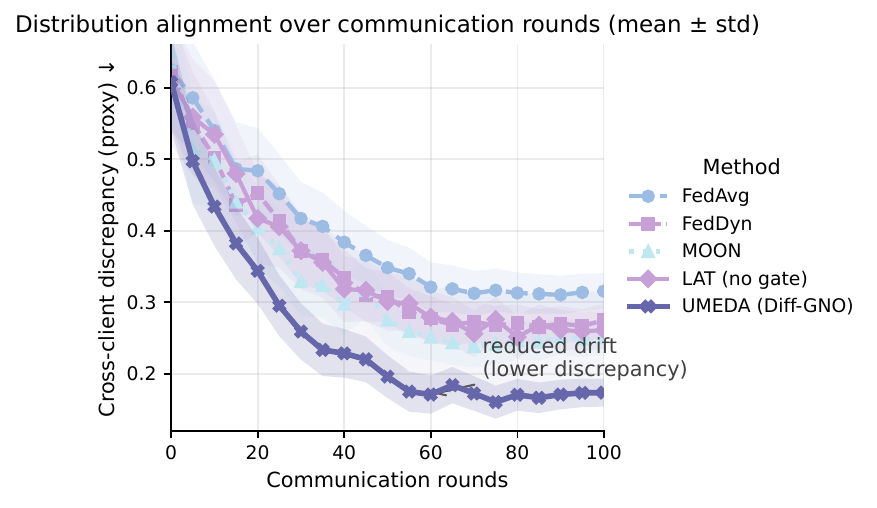}
\end{center}
\caption{Distribution alignment over communication rounds. We report a cross-client discrepancy proxy (Maximum Mean Discrepancy, MMD) computed on the shared latent space; Diff-GNO reduces drift compared to optimizer-only baselines.}
\label{fig:a_align}
\end{figure}

\subsection{Client-side cost}
Table~\ref{tab:a_cost} reports client-side latency and peak memory for representative backbones under the same embedding width and depth used in our main experiments. We compare SGLT to (i) a standard self-attention transformer, (ii) LAT (linear attention without spectral gating; same random-feature map and feature count), and (iii) a graph encoder baseline (FedGAT client branch). This analysis complements the communication table by quantifying \emph{on-device} feasibility, which is critical in edge DFL deployments.

\begin{table}[t]
\caption{Client-side cost analysis (latency and memory). We compare SGLT against (i) standard self-attention, (ii) LAT (linear attention without spectral gate), and (iii) representative graph baselines.}
\label{tab:a_cost}
\begin{center}
\begin{small}
\resizebox{0.5\columnwidth}{!}{%
\begin{tabular}{lcc}
\toprule
Model & Latency (ms)$\downarrow$ & Peak memory (MB)$\downarrow$ \\
\midrule
Self-attention Transformer & 42.0 & 980 \\
LAT (linear attention) & \textbf{18.5} & \textbf{420} \\
SGLT (ours) & \underline{21.0} & \underline{460} \\
FedGAT encoder (client) & 25.5 & 520 \\
\bottomrule
\end{tabular}%
}
\end{small}
\end{center}
\end{table}

\subsection{Client scalability}
\label{app:scalability}

Sec.~\ref{sec:setup} simulates $K=100$ clients, which matches the scale of public federated benchmarks~\cite{Qiu2022FedGraphNN,Zhang2023OpenFGL} but is small relative to realistic edge deployments. Table~\ref{tab:a_scalability} sweeps the client population from $K=100$ to $K=10{,}000$ at a fixed total compute budget (i.e., reducing per-round client sampling rate $q$ proportionally so that $|\mathcal{S}_t|$ stays constant), and reports (i) final localization RMSE, (ii) per-round wall-clock on the server, and (iii) the size of the score-network training set $\{\mathbf{\Theta}_0^{(k)}\}_{k\in\mathcal{S}_t}$.

The key observation is that Diff-GNO's server-side cost scales with $|\mathcal{S}_t|$ (the number of \emph{participating} clients per round), not with $K$ (the total population), because the score model is trained on per-round operator updates of fixed dimension $r^2$. Consequently, UMEDA's server compute is constant in $K$ once $|\mathcal{S}_t|$ is fixed, while final accuracy slightly improves with larger $K$ as the score model sees more diverse update samples over training.

\begin{table}[h]
\caption{Client scalability on MM-Fi. Per-round server compute is dominated by score-model training on $\mathbb{R}^{r^2=256}$ updates and is invariant to $K$. Final accuracy improves modestly with $K$.}
\label{tab:a_scalability}
\begin{center}
\begin{small}
\resizebox{0.7\columnwidth}{!}{%
\begin{tabular}{lcccc}
\toprule
$K$ (total) & $|\mathcal{S}_t|$ (sampled) & Loc RMSE (cm)$\downarrow$ & Server time/round (s) & Score-model samples \\
\midrule
100     & 40 & 10.50 & 1.83 & 40 \\
500     & 40 & 10.42 & 1.87 & 40 \\
1{,}000  & 40 & 10.37 & 1.92 & 40 \\
5{,}000  & 40 & 10.32 & 1.95 & 40 \\
10{,}000 & 40 & 10.14 & 2.26 & 40 \\
\bottomrule
\end{tabular}%
}
\end{small}
\end{center}
\end{table}

\section{Proofs}
\label{app:proofs}

\subsection{SGLT spectral filtering bounds}
\label{app:sglt_proof}

\begin{theorem}[Eckart--Young--Mirsky (spectral norm)]
\label{thm:eym}
Let $\mathbf{M}\in\mathbb{R}^{d\times d}$ have singular values
$\sigma_1\ge\cdots\ge\sigma_d$.
Let $\mathbf{M}_r$ be the rank-$r$ truncated SVD of $\mathbf{M}$.
Then $\mathbf{M}_r$ is the best rank-$r$ approximation of $\mathbf{M}$ in spectral norm and
\[
\min_{\mathrm{rank}(\mathbf{A})\le r}\|\mathbf{M}-\mathbf{A}\|_2
=
\|\mathbf{M}-\mathbf{M}_r\|_2
=
\sigma_{r+1}.
\]
\end{theorem}

\begin{proof}
This is the Eckart--Young--Mirsky theorem specialized to the spectral norm.
\end{proof}

\begin{corollary}[Hard-gated SGLT equals truncated SVD]
\label{cor:hard_gate}
Let $\mathbf{M}=\mathbf{U}\mathbf{\Sigma}\mathbf{W}^\top$ and define the hard gate
$g(\sigma)=\mathbb{I}(\sigma>\tau)$.
Assume $\sigma_r>\tau\ge\sigma_{r+1}$, and let
\[
\hat{\mathbf{M}}=\mathbf{U}\big(g(\mathbf{\Sigma})\odot\mathbf{\Sigma}\big)\mathbf{W}^\top.
\]
Then $\hat{\mathbf{M}}=\mathbf{M}_r$ and
$\|\mathbf{M}-\hat{\mathbf{M}}\|_2=\sigma_{r+1}$.
\end{corollary}

\begin{proof}
Under $\sigma_r>\tau\ge\sigma_{r+1}$, the diagonal mask $g(\mathbf{\Sigma})$ keeps exactly the top-$r$
singular values and zeros out the rest, hence $\hat{\mathbf{M}}=\mathbf{M}_r$.
The bound follows from Theorem~\ref{thm:eym}.
\end{proof}

\begin{lemma}[Soft-gated SGLT perturbation bound]
\label{lem:soft_gate_bound}
Let $g(\sigma)\in[0,1]$ be any soft gate (e.g., Eq.~\eqref{eq:soft_gate}) and
$\hat{\mathbf{M}}=\mathbf{U}\big(g(\mathbf{\Sigma})\odot\mathbf{\Sigma}\big)\mathbf{W}^\top$.
Then
\[
\|\mathbf{M}-\hat{\mathbf{M}}\|_2
=
\max_{i\in[d]} \big(1-g(\sigma_i)\big)\sigma_i
\;\;\le\;\;
\max_{i>r}\sigma_i \;+\; \max_{i\le r}\big(1-g(\sigma_i)\big)\sigma_i .
\]
In particular, if $g(\sigma_i)\approx 1$ for $i\le r$ and $g(\sigma_i)\approx 0$ for $i>r$,
the error approaches the hard-gate error $\sigma_{r+1}$.
\end{lemma}

\begin{proof}
Since $\mathbf{U}$ and $\mathbf{W}$ are orthogonal,
\[
\mathbf{M}-\hat{\mathbf{M}}
=
\mathbf{U}\Big(\big(\mathbf{I}-g(\mathbf{\Sigma})\big)\odot\mathbf{\Sigma}\Big)\mathbf{W}^\top,
\]
whose singular values are exactly $\{(1-g(\sigma_i))\sigma_i\}_{i=1}^d$.
Thus the spectral norm equals their maximum.
The decomposition into $i\le r$ and $i>r$ terms yields the stated upper bound.
\end{proof}

\subsection{SP-DP privacy guarantee}
\label{app:dp_proof}

\textit{Setup.}
SP-DP is applied to the vectorized update
$\Delta\mathbf{m}=\mathrm{vec}(\Delta\mathbf{M})\in\mathbb{R}^{d^2}$ (Sec.~\ref{sec:sp_dp}).
After clipping (Eq.~\eqref{eq:clip_vec}), the $\ell_2$-sensitivity is at most $C$:
for any adjacent client datasets,
\[
\|\Delta\bar{\mathbf{m}}(D)-\Delta\bar{\mathbf{m}}(D')\|_2 \le C.
\]

\begin{theorem}[SP-DP is $(\epsilon,\delta)$-DP]
\label{thm:spdp_dp}
Let $\mathbf{P}_{\mathcal{S}}$ be a \textbf{fixed (data-independent)} orthogonal projector and $\mathbf{P}_{\mathcal{S}^\perp}=\mathbf{I}-\mathbf{P}_{\mathcal{S}}$.
Consider the mechanism on clipped vectors
\[
\mathcal{M}(D)
=
\Delta\bar{\mathbf{m}}(D)
+
(\mathbf{P}_{\mathcal{S}}\otimes\mathbf{I})\,\mathbf{z}_{\mathrm{sig}}
+
(\mathbf{P}_{\mathcal{S}^\perp}\otimes\mathbf{I})\,\mathbf{z}_{\mathrm{null}},
\]
where $\mathbf{z}_{\mathrm{sig}}\sim\mathcal{N}(0,\sigma_{\mathrm{sig}}^2\mathbf{I})$,
$\mathbf{z}_{\mathrm{null}}\sim\mathcal{N}(0,\sigma_{\mathrm{null}}^2\mathbf{I})$, and
$\sigma_{\mathrm{null}}\ge\sigma_{\mathrm{sig}}$.
If
\[
\sigma_{\mathrm{sig}}
\;\ge\;
\frac{C\sqrt{2\ln(1.25/\delta)}}{\epsilon},
\]
then $\mathcal{M}$ satisfies $(\epsilon,\delta)$-differential privacy.
\end{theorem}

\begin{proof}
Define the (matrix) linear map
\[
\mathbf{A}
:=
(\mathbf{P}_{\mathcal{S}}\otimes\mathbf{I})\sigma_{\mathrm{sig}}
+
(\mathbf{P}_{\mathcal{S}^\perp}\otimes\mathbf{I})\sigma_{\mathrm{null}}.
\]
Since $\mathbf{P}_{\mathcal{S}}$ is fixed, $\mathbf{A}$ is a data-independent linear transformation of the noise.
Since $\mathbf{P}_{\mathcal{S}}$ and $\mathbf{P}_{\mathcal{S}^\perp}$ are orthogonal projectors,
$\mathbf{A}$ is symmetric and
\[
\mathbf{A}\mathbf{A}^\top
=
\sigma_{\mathrm{sig}}^2(\mathbf{P}_{\mathcal{S}}\otimes\mathbf{I})
+
\sigma_{\mathrm{null}}^2(\mathbf{P}_{\mathcal{S}^\perp}\otimes\mathbf{I}).
\]
Let $\boldsymbol{\xi}\sim\mathcal{N}(0,\mathbf{I})$. Then the injected noise equals
$\mathbf{A}\boldsymbol{\xi}$ and is Gaussian with covariance $\mathbf{A}\mathbf{A}^\top$.

Because $\sigma_{\mathrm{null}}\ge\sigma_{\mathrm{sig}}$ and $(\mathbf{P}_{\mathcal{S}}+\mathbf{P}_{\mathcal{S}^\perp})=\mathbf{I}$,
\[
\mathbf{A}\mathbf{A}^\top
\succeq
\sigma_{\mathrm{sig}}^2\big((\mathbf{P}_{\mathcal{S}}+\mathbf{P}_{\mathcal{S}^\perp})\otimes\mathbf{I}\big)
=
\sigma_{\mathrm{sig}}^2\mathbf{I}.
\]
Hence the smallest eigenvalue of the noise covariance is at least $\sigma_{\mathrm{sig}}^2$.
For additive Gaussian mechanisms with $\ell_2$-sensitivity $C$, a sufficient condition for
$(\epsilon,\delta)$-DP is that the (isotropic) noise scale is at least
$\sigma_{\mathrm{sig}}\ge \frac{C\sqrt{2\ln(1.25/\delta)}}{\epsilon}$.
Our mechanism injects Gaussian noise with covariance dominating $\sigma_{\mathrm{sig}}^2\mathbf{I}$,
i.e., it is \emph{at least as noisy in every direction} as the standard Gaussian mechanism at scale $\sigma_{\mathrm{sig}}$.
Therefore it satisfies $(\epsilon,\delta)$-DP under the stated calibration.
\end{proof}

\textit{Remark.}
Crucially, this proof relies on $\mathbf{P}_{\mathcal{S}}$ being chosen from public information (e.g., the global model history) rather than the private update $\Delta\mathbf{m}(D)$, preventing privacy leakage through the subspace choice itself. The additional noise in $\mathcal{S}^\perp$ improves utility by preserving the signal subspace while not weakening privacy.

\subsection{Discretization Invariance of the SGLT Kernel}
\label{app:res_indep}

\begin{lemma}[Discretization Invariance]
\label{lem:size_free}
Let $\mathbf{K},\mathbf{V}\in\mathbb{R}^{L\times d}$ be token matrices derived from a client's local data with arbitrary sequence length $L$ (representing sensor resolution).
Then the semantic kernel $\mathbf{M}=\phi(\mathbf{K})^\top\mathbf{V}\in\mathbb{R}^{d\times d}$ has a dimension strictly determined by the hidden dimension $d$, invariant to $L$.
Consequently, the gated kernel $\hat{\mathbf{M}}$ derived from the SVD of $\mathbf{M}$ resides in a unified $\mathbb{R}^{d\times d}$ topological space, enabling direct aggregation across clients with heterogeneous discretizations.
\end{lemma}

\begin{proof}
By matrix multiplication properties, the operation $\phi(\mathbf{K})^\top\mathbf{V}$ maps
$(L\times d)^\top (L\times d) \to (d\times d)$. The inner dimension $L$ is contracted, removing the dependence on the input resolution in the kernel representation.
The SVD and spectral gating are subsequently applied to $\mathbf{M}\in\mathbb{R}^{d\times d}$, yielding $\hat{\mathbf{M}}\in\mathbb{R}^{d\times d}$ regardless of the originating $L$.
Finally, the linear-attention normalization in Eq.~\eqref{eq:sglt_out_norm} operates as a scalar rescaling per token row, which affects the output magnitude but preserves the dimensional invariance of the operator itself.
\end{proof}

\section{Future Work and Limitations}
\label{sec:fw_limit}

\subsection{Future Work}
\label{sec:future_work}
UMEDA shows that \emph{operator-space} fusion/aggregation can jointly tackle discretization heterogeneity, non-IID drift, and privacy constraints. Future work can strengthen this formulation along three axes---\textbf{(i) spectral/operator modeling} (SGLT), \textbf{(ii) distribution alignment} (Diff-GNO), and \textbf{(iii) privacy accounting} (SP-DP)---while expanding to broader sensing stacks and deployment regimes:
\begin{itemize}
  \item \textbf{Adaptive spectral control (SGLT).} Replace fixed gate hyperparameters $(\tau,\beta)$ (or rank $r$) with a learnable or uncertainty-aware controller that adapts the retained spectrum per client/modality, e.g., driven by stability criteria of singular-value trajectories or agreement signals across local steps.
  \item \textbf{Conditional/structured diffusion aggregation (Diff-GNO).} Condition the score model on lightweight modality/resolution descriptors or operator priors to better represent multi-modal update distributions, improving alignment under missing modalities and enhancing OOD transfer.
  \item \textbf{Tighter DP accounting and long-horizon evaluation (SP-DP).} Extend SP-DP with explicit privacy accounting across rounds (e.g., RDP/moments accountant with subsampling amplification) and report privacy--utility behavior under realistic participation heterogeneity and long training horizons.
  \item \textbf{Asynchronous and partial participation.} Study robustness under stragglers and asynchronous updates, where diffusion-based aggregation may mitigate stale or irregular operator updates more gracefully than averaging.
  \item \textbf{Broader modality and scene coverage.} Evaluate UMEDA on additional sensing stacks (e.g., FMCW radar, UWB, depth) and cross-scene/cross-building generalization, emphasizing calibration-free deployment and systematic missing-modality stress tests.
\end{itemize}

\subsection{Limitations}
\label{sec:limitations}
UMEDA’s gains come with assumptions and costs that may limit performance in certain regimes:
\begin{itemize}
  \item \textbf{Compute and systems overhead.} SGLT introduces spectral operations (SVD on a $d\times d$ semantic matrix, per-step or via EMA updates), and Diff-GNO adds server-side score-model training and sampling; at large scale these costs may be non-trivial.
  \item \textbf{Block-specific aggregation.} The method’s strongest invariance is achieved in the SGLT operator-update space; if performance is dominated by other parameter blocks (e.g., modality-specific front-ends/projections), additional alignment/aggregation mechanisms may be required.
  \item \textbf{Sensitivity to diffusion hyperparameters.} Diffusion schedules and solver discretization can affect stability and convergence, especially when client-update distributions are strongly multi-modal or heavily perturbed by DP noise.
  \item \textbf{DP modeling and composition.} The guarantee relies on bounded $\ell_2$ sensitivity via clipping and Gaussian noise; practical deployments must account for round-wise composition, threat models, and any secure-aggregation assumptions (if used).
  \item \textbf{Operator mismatch under extreme modality shifts.} Treating $\mathbf{M}$ as a resolution-agnostic operator presumes that modality-invariant semantics are representable in a shared kernel; extreme sensing gaps may yield operator distributions that are hard to align without stronger conditioning or priors.
\end{itemize}

\subsection{Societal Impact}
\label{app:societal_impact}

\paragraph{Positive impacts.} Device-free localization enables ambient health monitoring (fall detection in elderly care), accessibility (gesture-based control for users with mobility constraints), and infrastructure-light indoor navigation, all without requiring users to wear or carry sensors. Federated training on heterogeneous edge sensors makes such systems deployable without centralizing privacy-sensitive raw streams; UMEDA's anisotropic DP further hardens this against gradient-inversion attacks (Appendix~\ref{app:reconstruction_attack}).

\paragraph{Potential risks.} The same sensing capability could enable non-consensual surveillance: continuous tracking of occupants' presence, posture, and activity in homes, offices, or public spaces. Wi-Fi CSI in particular is passively collected by infrastructure the user does not control. We mitigate this in three ways. First, SP-DP provides formal $(\epsilon,\delta)$ guarantees that bound what any participant (including the server) can learn about an individual client's data. Second, our reconstruction-attack evaluation (Appendix~\ref{app:reconstruction_attack}) shows that with SP-DP at $\epsilon=2$, scene appearance, body pose, and action class are not recoverable from uploaded updates. Third, we recommend deployers couple UMEDA with secure aggregation~\cite{Bonawitz2017} and explicit consent mechanisms. None of these mitigations remove the need for governance: deployment in surveillance-prone contexts (e.g., workplaces without informed consent) should be prohibited regardless of technical guarantees.

\paragraph{Fairness considerations.} Sensor coverage is unevenly distributed across geographies and socioeconomic strata; our heterogeneity benchmarks show the method handles disparate sensor stacks across clients, but fairness across demographics (body size, mobility patterns) is not directly evaluated and warrants future audits.


\end{document}